\documentclass{article}

\usepackage[margin = 0.9in]{geometry}
\usepackage{newtxtext}
\usepackage[round,sort,compress]{natbib}
\usepackage{fancyhdr}
\usepackage[parfill]{parskip}

\usepackage{amsmath,amsfonts,bm,amssymb,mathtools}

\def\eqref#1{(\ref{#1})}
\def\1{\bm{1}}

\def\vzero{{\bm{0}}}

\def\ve{{\bm{e}}}

\def\vu{{\bm{u}}}
\def\vv{{\bm{v}}}
\def\vw{{\bm{w}}}
\def\vx{{\bm{x}}}
\def\vy{{\bm{y}}}

\def\mI{{\bm{I}}}

\DeclareMathAlphabet{\mathsfit}{\encodingdefault}{\sfdefault}{m}{sl}
\SetMathAlphabet{\mathsfit}{bold}{\encodingdefault}{\sfdefault}{bx}{n}

\def\gL{{\mathcal{L}}}

\def\gN{{\mathcal{N}}}
\def\gO{{\mathcal{O}}}

\def\sR{{\mathbb{R}}}

\usepackage{dsfont}
\newcommand{\indicator}{\mathds{1}}

\usepackage{xcolor,hyperref}
\definecolor{darkblue}{rgb}{0.0,0.0,0.65}
\definecolor{darkred}{rgb}{0.65,0.0,0.0}
\definecolor{darkgreen}{rgb}{0.0,0.5,0.0}
\definecolor{tab:blue}{RGB}{31,119,180}  % 1f77b4
\definecolor{tab:red}{RGB}{214,39,40}  % d62728
\definecolor{tab:green}{RGB}{44,160,44}  % 2ca02c
\definecolor{tab:orange}{RGB}{255,127,14}  % ff7f0e
\hypersetup{
	colorlinks = true,
	citecolor  = darkblue,
	linkcolor  = darkred,
	filecolor  = darkblue,
	urlcolor   = darkblue,
}

\usepackage{url}
\usepackage{amsmath,amssymb,amsfonts,amsthm}
\usepackage{float}
\usepackage{subfigure}
\usepackage{lipsum}
\usepackage{enumitem}
\usepackage[capitalize,noabbrev]{cleveref}

\usepackage{caption}

\theoremstyle{plain}

\newtheorem{theorem}{Theorem}[section]

\newtheorem{lemma}{Lemma}[section]

\newtheorem{remark}{Remark}[section]
\newtheorem{hypothesis}{Hypothesis}

\usepackage[utf8]{inputenc} % allow utf-8 input
\usepackage[T1]{fontenc}    % use 8-bit T1 fonts
\usepackage{hyperref}       % hyperlinks
\usepackage{url}            % simple URL typesetting
\usepackage{booktabs}       % professional-quality tables
\usepackage{amsfonts}       % blackboard math symbols
\usepackage{nicefrac}       % compact symbols for 1/2, etc.
\usepackage{microtype}      % microtypography
\usepackage{xcolor}         % colors

\usepackage{wrapfig}
\newcommand{\rulesep}{\unskip\ \vrule\ }

\allowdisplaybreaks

\title{\textbf{Gradient Descent with Polyak's Momentum Finds Flatter Minima via Large Catapults}}

\date{}
\author{
  Prin Phunyaphibarn\thanks{Equal contributions} \thanks{Work done as an undergraduate intern at KAIST AI.} \\
  KAIST Math\\
  \texttt{\small prin10517@kaist.ac.kr} \\
  \and
  Junghyun Lee\footnotemark[1] \\
  KAIST AI\\
  \texttt{\small jh\_lee00@kaist.ac.kr} \\
  \and
  Bohan Wang \\
  USTC \& Microsoft Research Asia\\
  \texttt{\small bhwangfy@gmail.com} \\
  \and
  Huishuai Zhang \\
  Peking University \\
  \texttt{\small zhanghuishuai@pku.edu.cn} \\
  \and
  Chulhee Yun \\
  KAIST AI \\
  \texttt{\small chulhee.yun@kaist.ac.kr} \\
}

\begin{document}

\maketitle

\begin{abstract}
Although gradient descent with Polyak's momentum is widely used in modern machine and deep learning, a concrete understanding of its effects on the training trajectory remains elusive.
In this work, we empirically show that for linear diagonal networks and nonlinear neural networks, momentum gradient descent with a large learning rate displays large catapults, driving the iterates towards much flatter minima than those found by gradient descent.
We hypothesize that the large catapult is caused by momentum ``prolonging'' the self-stabilization effect~\citep{damian2023stabilization}.
We provide theoretical and empirical support for our hypothesis in a simple toy example and empirical evidence supporting our hypothesis for linear diagonal networks.
\end{abstract}

\section{Introduction}
Although momentum is one of the most widely used and crucial component for training modern neural networks~\citep{sutskever2013momentum}, our understanding of the effects of momentum on neural network training dynamics is still lacking.
Throughout the paper, for a loss function $\mathcal{L}(\vw)$, we consider Polyak's heavy-ball momentum (PHB; \citet{polyak1964momentum}) method given as follows:
\begin{equation}
\label{eqn:momentum}
    \vw_{t+1} = \vw_t - \eta_t \nabla \mathcal{L}(\vw_t) + \beta (\vw_t - \vw_{t-1}),
\end{equation}
where $\eta_t$ denotes the learning rate that may change over time and $\beta \in [0,1)$ is the momentum parameter where $\beta=0$ for gradient descent (GD). From here on, we will refer to Polyak's momentum simply as ``momentum'' or ``PHB''.
For momentum, we are not even sure yet why and when it accelerates (stochastic) gradient descent~\citep{fu2023momentum,ganesh2023momentum,kidambi2018momentum}, and perhaps more importantly, how it changes the implicit bias and thus the resulting model's generalization capability~\citep{wang2023momentum,wang2022implicit,ghosh2023implicit,jelassi2022momentum}.

An increasing number of works also consider training dynamics under the large learning rate regime. One such work closely related to ours is the \emph{catapult mechanism} introduced by \citet{lewkowycz2020catapult}. Classical optimization theory suggests that if the \emph{sharpness} $S := \lambda_{max}(\nabla^2 \mathcal{L})$, defined as the maximum eigenvalue of the Hessian of the training loss, is larger than a certain threshold, then training should be unstable and divergent. We refer to this threshold as the \emph{maximum stable sharpness (MSS)}, and it is equal to $2/\eta$ for GD and $2(1+\beta)/\eta$ for momentum~\citep{goh2017momentum}. Thus, conventional wisdom indicates that the learning rate $\eta$ should be selected so that $S < 2/\eta$. However, \citet{lewkowycz2020catapult} show that rather than diverging, GD initialized with a large learning rate $\eta$ satisfying $2/S < \eta  < 4/S$ displays a loss spike and a simultaneous sharpness drop, which drives the iterates towards a flat region with stable sharpness $S < 2/\eta$.
Following this observation, throughout this paper, we define a {\bf catapult} as \emph{a drastic sharpness reduction coupled with a single spike in the training loss}.
Although \citet{lewkowycz2020catapult} study the catapult phenomenon in the early training dynamics with a large constant learning rate, some recent studies~\cite{zhu2023catapult} consider experimental settings where the learning rate is initialized to be small ($\eta < 2/S$) and increases over training to \emph{induce} catapults. These works suggest that strange and interesting phenomena manifest in the large learning rate regime.

A phenomenon related to the catapult phenomenon that has garnered much attention in recent years is the \emph{edge of stability} or \emph{EoS} phenomenon. \citet{cohen2021stability} report two surprising phenomena that consistently occur during neural network training. First, when the sharpness is below the MSS, the sharpness tends to increase due to a phenomenon known as \emph{progressive sharpening (PS)}. Generally, progressive sharpening will increase the sharpness until it is above the MSS. However, instead of diverging as classical optimization theory suggests, the sharpness oscillates around the MSS while the training loss decreases non-monotonically. In a sense, the EoS phenomenon can be viewed as a cycle of a catapult and PS: once PS drives the sharpness above the MSS, a catapult occurs and decreases the sharpness below the MSS while the training loss spikes. Once the sharpness is lower than the MSS, PS increases it again, and the cycle continues. This cycle can be viewed together as sharpness oscillation and a non-monotonic decrease in training loss as reported by \citet{cohen2021stability}.
Subsequent works have formalized this intuition~\citep{damian2023stabilization,agarwala2023eos}. For more discussion on related works, please refer to Appendix~\ref{app:related-works}.
% Thus, it appears that selecting the learning rate based on the sharpness generally does not stabilize training dynamics because the sharpness will increase and remain on the edge of stability.

\paragraph{Contributions.}
In this work, we explore a phenomenon that arises by introducing momentum to the large learning rate regime. We provide experimental evidence suggesting that PHB in the large learning rate regime exhibits a phase of much steeper {\it sharpness drop} than GD with the same warmup, which we call {\bf large catapults}. These provide an implicit bias effect of driving the iterates towards much flatter minima compared to GD. Our contributions are as follows:
\begin{itemize}[leftmargin=10pt]
    \item As a motivating example, we train linear diagonal networks~\citep{woodworth2020kernel,nacson2022stepsize} with different learning rate schedules and initialization scales $\alpha$'s and demonstrate that PHB exhibits an implicit bias drastically different from GD, obtaining low test loss even at moderate $\alpha$. We identify the underlying cause for this difference and discover that \emph{large catapults} occur when using momentum with learning rate warmup. We also show empirically that this phenomenon can be observed for general nonlinear neural networks (Section~\ref{sec:catapults}).
    
    \item Drawing inspiration from the self-stabilization mechanism~\citep{damian2023stabilization} originally proposed in the EoS literature, we hypothesize that the large catapults occur because \emph{momentum prolongs the reduction of sharpness and movement along the ``unstable direction''} (Section~\ref{sec:prolong}).
    For simple networks (a ReLU scalar network and linear diagonal networks), we provide theoretical and empirical evidence supporting our hypothesis (Section~\ref{sec:toy}).
    % an explanation for the large catapults using {\color{red}a Lyapunov stability analysis} of PHB, which may be of independent interest (Section~\ref{sec:prolong}).
    % \item Finally, we confirm that the phenomenon of large catapults persists in more practical scenarios by experimenting with fully-connected networks and ResNets on 
    % subsets of the CIFAR10 dataset (Section~\ref{sec:modern-arch}).%a synthetic dataset as well as a subset of the CIFAR10 dataset (Section~\ref{sec:modern-arch}).
\end{itemize}

% We then discuss why and how such phenomena occur through carefully crafted experimental verifications as well as some theoretical intuitions, inspired by the \emph{self-stabilization} of GD~\citep{damian2023stabilization} and {\it catapult mechanism} of (S)GD~\citep{zhu2023catapult,lewkowycz2020catapult,meltzer2023catapult}.

% {\color{red}\paragraph{Paper Organization.}
% Section 2 presents ....,
% ...
% For fair comparison, we always match the MSS of GD and PHB by properly rescaling the momentum learning rates $\eta_{PHB} = (1+\beta) \eta_{GD}$ so that the MSS match. 

% }
\section{Large Catapults in Momentum Gradient Descent}
\label{sec:catapults}

\subsection{Motivating Example: Linear Diagonal Networks}
\label{sec:ldn}

% The diagonal neural network \citep{woodworth2020kernel,pesme2021sgd,nacson2022stepsize} as it is arguably the simplest 
% \cite{woodworth2020kernel} showed that for gradient flow, the scale of initialization induces a separation between kernel and rich regimes; in the latter regime, the algorithm displays sparsity-inducing $\ell_1$ regularization; \cite{nacson2022stepsize} empirically showed that for gradient descent with finite step-size, the scale of the step-size induces a similar separation; \cite{pesme2021sgd} quantified the implicit regularization due to stochasticity when stochastic gradient flow is used.
The linear diagonal network (LDN) is known to be one of the simplest non-linear models that display rich and non-trivial implicit bias~\citep{woodworth2020kernel} while still being mathematically tractable.
Here, we focus on the depth-2 linear diagonal network defined by
\begin{equation}
    f(\vx; \vu, \vv) := \langle \vu \odot \vu - \vv \odot \vv, \vx \rangle 
    %= \langle \vw, \vx \rangle
    , \quad \vx, \vu, \vv \in \sR^d,
\end{equation}
where $\vx$ is the input vector, $(\vu, \vv)$ are the trainable parameters, and $\vw := \vu \odot \vu - \vv \odot \vv$ is the linear coefficient vector.
% In sparse regression problems where the true $\vw$ is assumed to be sparse, it is known that gradient-based optimization has an implicit regularization effect for LDNs that interpolates between minimum $\ell_1$ and $\ell_2$-norm solutions, depending on the scale of initialization \citep{woodworth2020kernel}.

\paragraph{Implicit Bias of LDN.} The training dynamics and implicit bias of the LDN have been rigorously investigated in the literature. Still, most works focus on the continuous regime with vanishing learning rate, such as the (stochastic) gradient flow~\citep{woodworth2020kernel,azulay2021initialization,moroshko2020initialization,pesme2023hopping,pesme2021sgd}.
Although \cite{papazov2024momentum} and \cite{lyu2024momentum} study the impact of momentum on training dynamics, they also focus on the flow regime, which cannot capture phenomena such as the edge of stability or catapults that are inherently unique to the large learning rate regime.
% \cite{papazov2024momentum} derive implicit regularization terms specifically for LDNs trained with momentum gradient flow whereas our phenomenon holds more generally with various architectures, loss, and datasets.  
Recently, there has been some progress for (S)GD with finite step sizes~\citep{nacson2022stepsize,even2023diagonal}, but they do not consider momentum at all.

\citet{woodworth2020kernel} prove that for the sparse regression problem where the true $\vw_\star$ is assumed to be sparse, the gradient flow for LDN initialized at $\vu_0 = \vv_0 = \alpha\cdot\bm{1}$ converges to the minimum $\ell_1$-norm solution as $\alpha \rightarrow 0$ and minimum $\ell_2$-norm solution as $\alpha \rightarrow \infty$.
It is well-known that the minimum $\ell_1$-norm solution has a better generalization capability, due to sparsity of $\vw_\star$.
% If the underlying problem has an intrinsic sparse structure (i.e., if the ground-truth $\vw_*$ is sparse), the minimum $\ell_1$-norm solution has a better generalization capability.
Subsequently, \cite{nacson2022stepsize} show that GD with finite learning rate consistently recovers solutions with smaller test loss, even for initializations with large $\alpha$'s, as shown in Figure~\ref{fig:LDN-gd}.
% Since the properties of vanilla GD are well-known in LDNs, it is only natural to first consider the effect of momentum and learning rate warmup in LDNs to get a clear picture.

\vspace{-5pt}
\paragraph{Experimental Setting.}
We investigate the implicit bias effect of momentum in linear diagonal networks by adding momentum to the sparse regression experiment in \citet{nacson2022stepsize}.
% $\vy_i \sim \gN_d(\langle \bm\theta^\star, \vx_n \rangle, \sigma^2 \mI)$
Throughout this section, we mostly follow their experimental settings.
The training set is $\{(\vx_n, \vy_n)\}_{n=1}^N$ generated as $\vx_n \sim \gN(\bm\mu, \sigma^2 \mI_d)$ and $\vy_n = \langle \vw_\star, \vx_n \rangle$, where $N = 50, \sigma^2 = 5, d = 100, \bm\mu = 5\cdot\bm{1}$ and $\vw_\star = (\delta_{i \leq 5} / \sqrt{5})_i$.
We use the mean squared error (MSE) loss, and we initialize $\vu_0 = \vv_0 = \alpha\cdot\bm{1}$, where $\alpha \in [0, 0.3]$ is the initialization scale.
% For each $\eta_f \in \{ 1 \cdot 10^{-4}, 2 \cdot 10^{-4}, 4 \cdot 10^{-4}, 7 \cdot 10^{-4}, 1 \cdot 10^{-3}, 1.3 \cdot 10^{-3}, 1.8 \cdot 10^{-3}, 2.4 \cdot 10^{-3}, 3.2 \cdot 10^{-3}, 4.2 \cdot 10^{-3}, 5.6 \cdot 10^{-3}, 7.5 \cdot 10^{-3} \}$, 
Following \cite{nacson2022stepsize}, for each $\eta_f$, we use linear warmup (linearly increasing the learning rate) for the first $\eta_f \cdot 10^6$ steps, starting from $\eta_i = 10^{-8}$.
For the momentum parameter of PHB, we use $\beta = 0.9$.

\begin{remark}[Learning Rate Warmup]
    Learning rate warmup is a standard practice in deep learning~\citep{goyal2017warmup}, with known benefits such as variance reduction~\citep{gotmare2018warmup,liu2020warmup} and preconditioning~\citep{gilmer2022stability}.
    We and \cite{nacson2022stepsize} consider learning rate warmup to avoid possible instabilities from using large learning rates.
\end{remark}
% ,goyal2017warmup

% In this paper, we mainly consider linear learning rate warmup, henceforth referred to as ``linear warmup.''

%, which is the default parameter for PyTorch implementation of PHB\footnote{One subtle different is that in the PyTorch implementation the effective momentum coefficient is $\beta \eta$, not $\beta$. Here, we follow the PHB formulation as discussed in \cite{sutskever2013momentum} - \eqref{eqn:momentum}.}.
% In the flow regime ($\eta = ?$), the ``speed'' of implicit bias, i.e., speed of test loss decreasing as one decreases $\alpha$, is faster for PHB than GD, but overall they behave similarly.
% Things are very different for a finite learning rate regime ($\eta = ?$).
% \subsection{Momentum Changes the Implicit Bias of LDN }
\paragraph{Results.}
Over the varying initialization scales $\alpha$ and (final) learning rates $\eta_f$, we report the final resulting test losses in Figure~\ref{fig:LDN}.
Note the striking difference between (a) GD and (b) PHB.
Compared to GD, whose final test loss saturates after some $\alpha$ as reported in \citet{nacson2022stepsize}, it is clear that
\begin{center}
    {\it PHB displays a fundamentally different implicit bias in the large learning rate regime.}
\end{center}
For PHB, the final test loss initially increases monotonically as $\alpha$ increases like GD. However, once $\alpha$ becomes larger than some threshold $\bar{\alpha}(\eta_f)$ (dependent on $\eta_f$), the test loss sharply drops close to zero.
% This phase transition is absent when the LDN is trained using GD.
We also note that $\bar{\alpha}(\eta_f)$ appears to decrease as $\eta_f$ increases.\footnote{In fact, we can see in Figure~\ref{fig:LDN-main-phb} that the final test loss slightly increases (in a noisy fashion) with $\alpha$ after the sharp drop at $\bar \alpha(\eta_f)$. We attribute this phenomenon to \emph{overshooting}; see Appendix~\ref{app:overshooting} for more discussion.}
% Compared to GD, which displays a rather monotonic relationship between $\alpha$ and test loss, i.e., larger $\alpha$ roughly corresponds to larger test loss (up to saturation), it seems clear that in the presence of warmup, PHB has a fundamentally different behavior from GD in the finite learning rate regime.

% {\color{red}\bf One can hypothesize that due to enlarging MSS, adding momentum sometimes averts a catapult that is beneficial to the performance. Thus it is favored that the work can give more discussions and criterions of how adding momentum changes catapults and when momentum helps test performance. Besides, theoretical insights on the effects of initialization scales is also favored.}

\begin{figure*}[!t]
\centering
\subfigure[GD]{
\includegraphics[width=0.43\textwidth]{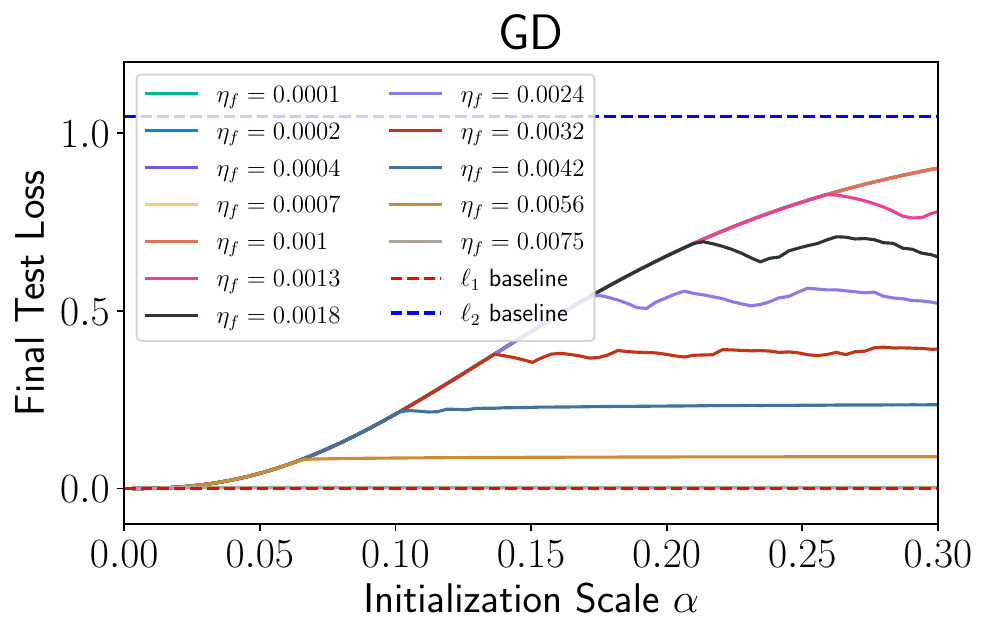}
\label{fig:LDN-gd}
}
\subfigure[PHB]{
\includegraphics[width=0.43\textwidth]{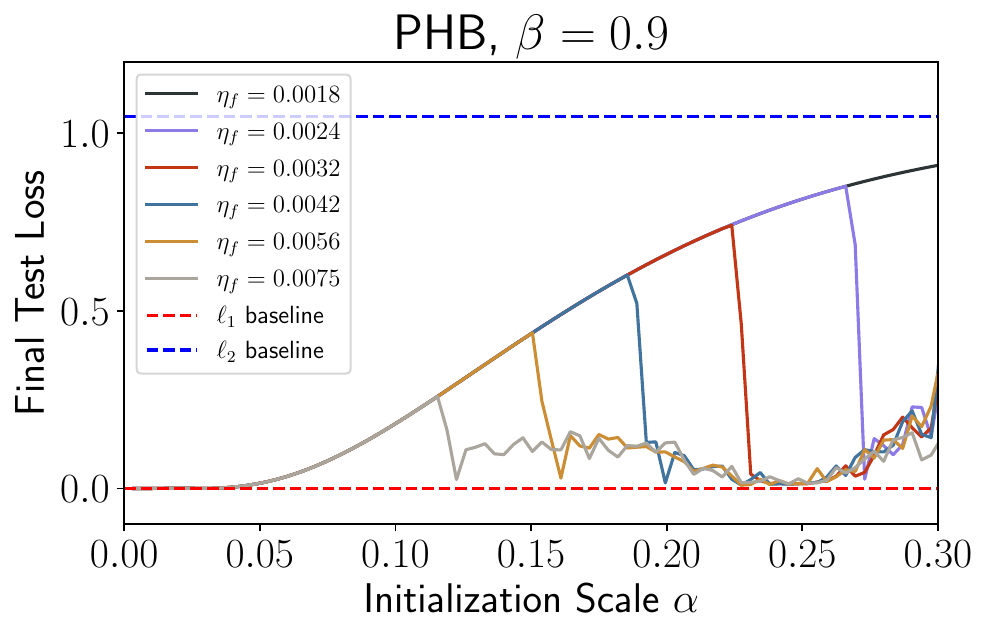}
\label{fig:LDN-main-phb}
}

\subfigure[GD Sharpness]{
\includegraphics[width=0.225\textwidth]{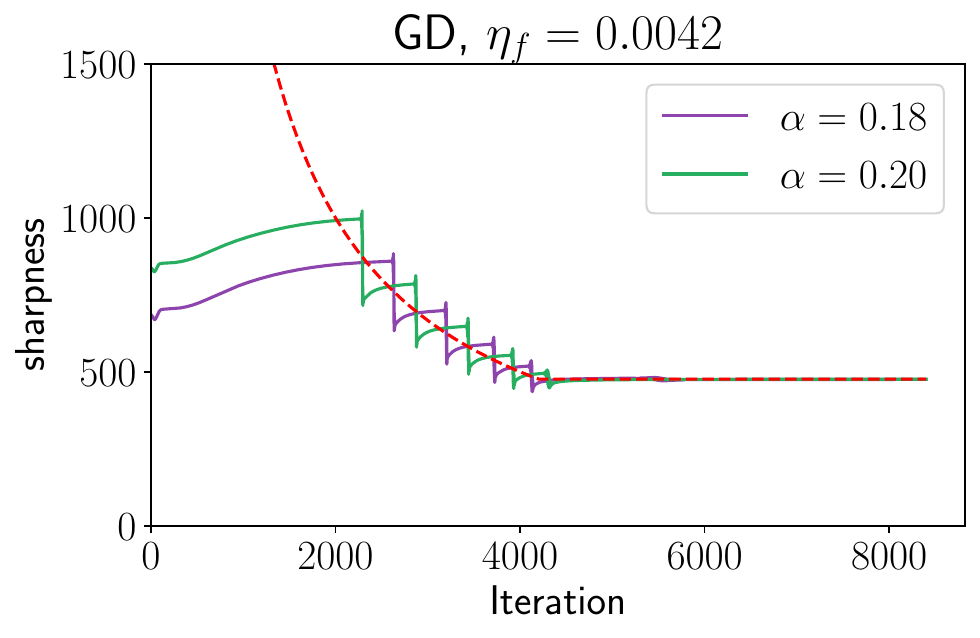}
\label{fig:LDN-sharpness-gd}
}
\subfigure[GD Train Loss]{
\includegraphics[width=0.225\textwidth]{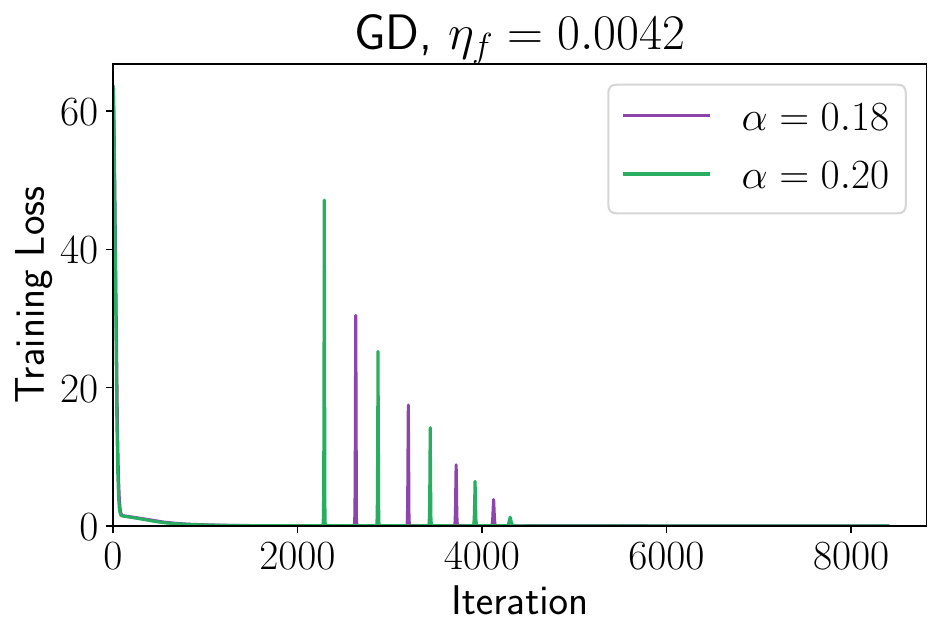}
\label{fig:LDN-train-loss-alpha-gd}
}
\subfigure[PHB Sharpness]{
\includegraphics[width=0.225\textwidth]{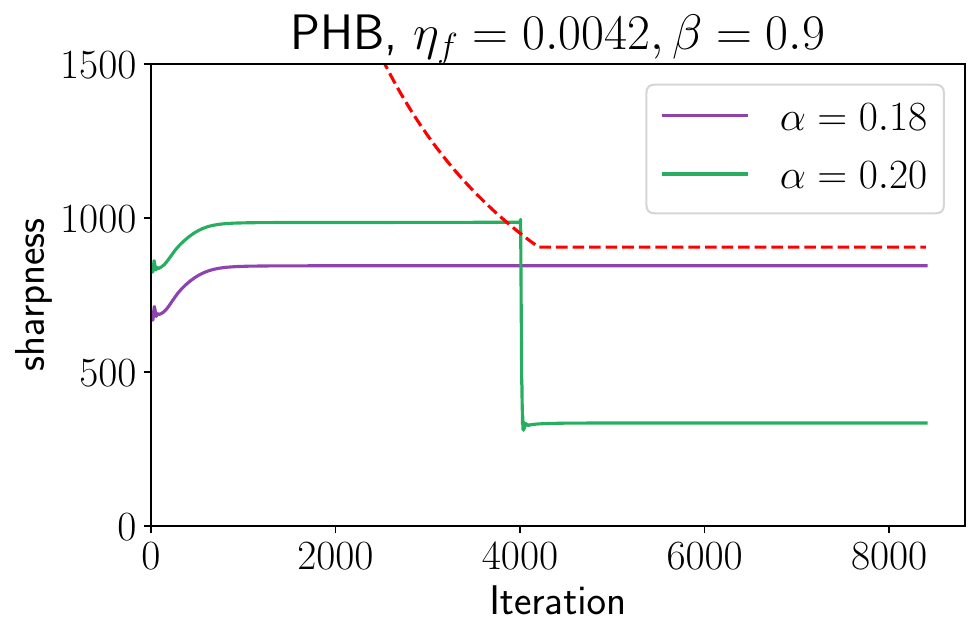}
\label{fig:LDN-sharpness-phb}
}
\subfigure[PHB Train Loss]{
\includegraphics[width=0.225\textwidth]{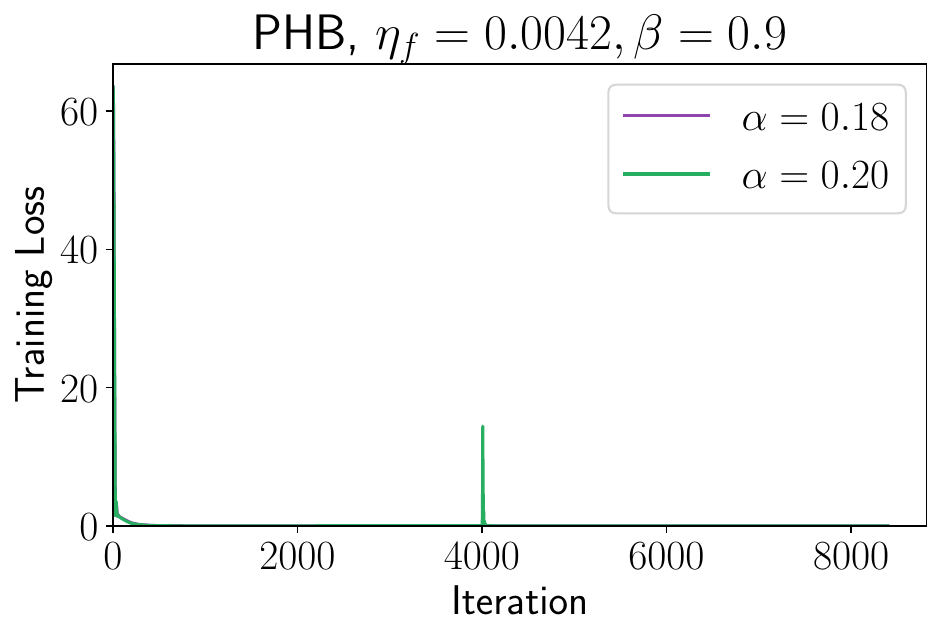}
\label{fig:LDN-train-loss-alpha-phb}
}

\caption{Experiments following the same setting as \citet{nacson2022stepsize}. In (a) and (b), "$\ell_1$ baseline" and "$\ell_2$ baseline" respectively stand for the solution with the minimal $\ell_1$ norm and the solution with the minimal $\ell_2$ norm to the regression problem. We use $\beta=0.9$ for PHB.}
\label{fig:LDN}
\end{figure*}

\paragraph{Momentum Induces Larger Catapults.}
% \subsection{Momentum with LR Warmup Induces Greater Sharpness Reduction}
% To further investigate our phenomenon, we need to recall that there are emergent phenomena in the large learning regime. %\citet{cohen2021stability} show that PHB exhibits unstable behaviors when the sharpness exceeds the MSS $\frac{2(1 + \beta)}{\eta}$.
% Indeed, it is known that when the learning rate is sufficiently large such that the iterates' sharpness goes \emph{above} the MSS, then several interesting phenomena such as edge of stability \citep{cohen2021stability} and catapult~\citep{lewkowycz2020catapult} arise, as discussed earlier.
% Throughout this paper, we refer to \emph{catapult} as a sharp increase in loss, followed by a decrease that forms a single spike in the training loss, coupled with a rapid sharpness reduction.

To find the cause for our observed phenomena, we plot the evolution of sharpness for PHB in Figure~\ref{fig:LDN-sharpness-phb} at the $\alpha$ and $\eta$ at which the test loss suddenly drops.
Note that as the warmup proceeds, the MSS decreases monotonically, and as soon as the sharpness of the iterates ``touches'' the MSS curve, it goes through a rapid sharpness reduction, coupled with a loss spike (Figure~\ref{fig:LDN-train-loss-alpha-phb}).
The sharpness reduction of PHB is so drastic that the final sharpness is well below the MSS of the final learning rate.
This is in contrast to GD which goes through an incremental, step-wise sharpness reduction (Figure~\ref{fig:LDN-sharpness-gd}) with multiple loss spikes (Figure~\ref{fig:LDN-train-loss-alpha-gd}), and the final sharpness stays just below the MSS corresponding to the final learning rate; this phenomenon for SGD was also observed in \citet{zhu2023catapult}.
Here, one could make an educated guess that momentum induces much larger catapults that bias the solution towards flatter minima.

One observation is that momentum does not improve test loss over GD across \emph{all} settings (e.g., initialization), as shown in Figure~\ref{fig:LDN-main-phb}. A possible explanation is since the MSS $\frac{2(1+\beta)}{\eta}$ for PHB is higher than the MSS $\frac{2}{\eta}$  of GD with the same $\eta$, the $\alpha$ that causes a catapult for GD may not for PHB, resulting in no improvement, but if a catapult does occur, the sharpness reduction is much more drastic for PHB than GD. With this observation in mind, from this point on, whenever we compare the trajectory of GD and PHB starting from the same initialization, we always match the MSS of GD and PHB by properly rescaling the momentum learning rates $\eta_{PHB} = (1+\beta) \eta_{GD}$ for fair comparison. We only specify the GD learning rate in the text as $\eta$.

\begin{remark}[Role of Warmup in Large Catapults]
\label{rmk:roleofwarmup}
    Although we used linear warmup schedule in the experiments, warmup is not a strict requirement for (large) catapults to occur.
    Indeed, we observe that catapults consistently occur as long as (1) the iterate is initially close to a global minimum and (2) the learning rate is large enough so that the sharpness is slightly above the MSS but small enough so that the iterates do not completely diverge.
    Linear warmup provides a natural way to satisfy these two criteria, whereas 
    using large constant learning rates from the beginning usually results in severe instability.
    %It is difficult to satisfy these two conditions without using warmup, as training with a large constant learning rate from the beginning results in severe instability.
    We remark that one could use different scheduling to induce the catapults; see Appendix~\ref{app:warmup-necessity} for ablations on different warm-up schedules.  
\end{remark}

\begin{figure*}[!t]
    \centering
    \subfigure[Width-200 FCN, $\eta_i=0.001$, $\eta_f=0.05$]{
        \includegraphics[width=0.29\textwidth]{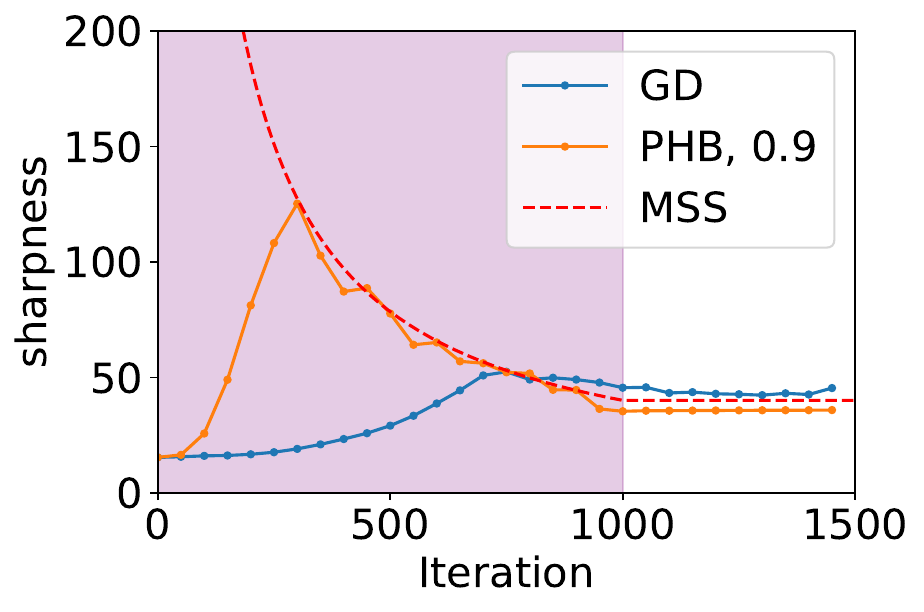}
    }
    \hfill
    \subfigure[Width-1000 FCN, $\eta_i=0.01$, $\eta_f=0.1$]{
        \includegraphics[width=0.29\textwidth]{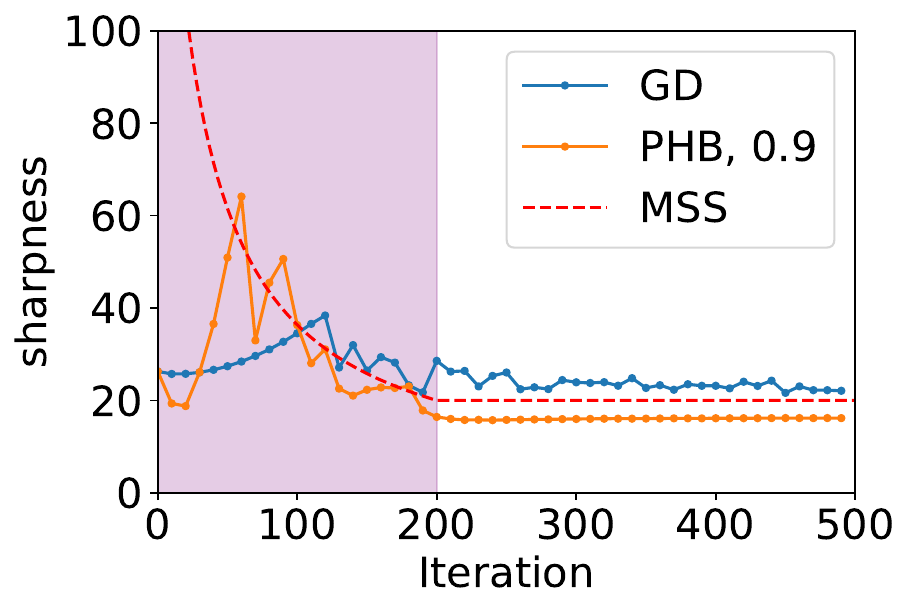}
    }
    \hfill
    \rulesep
    \subfigure[Width-100 FCN, $\eta_i=0.001$, $\eta_f=0.4$]{
    \includegraphics[width=0.29\textwidth]{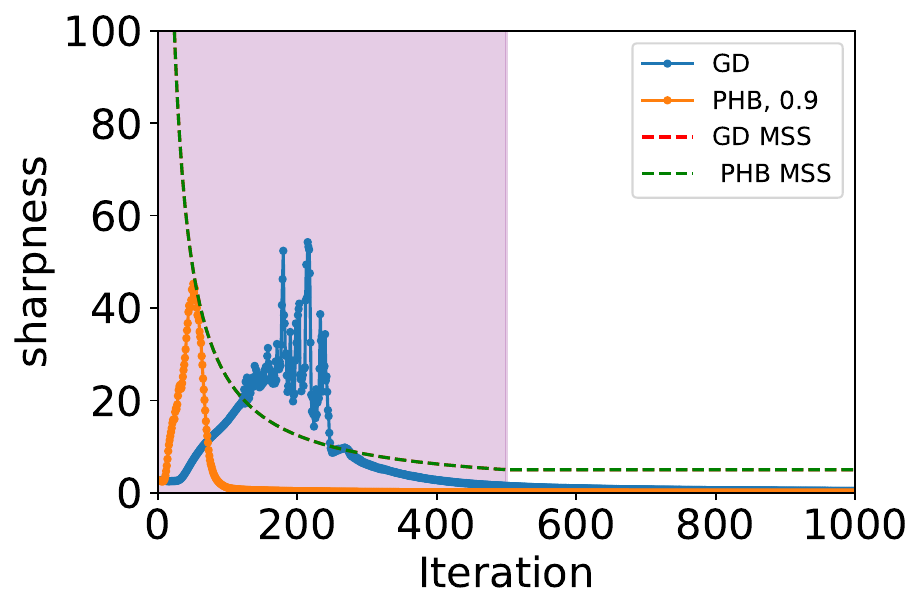}
    } \\
    \subfigure[Width-1000 FCN, $\eta_i=0.001$, $\eta_f=0.02$]{
        \includegraphics[width=0.29\textwidth]{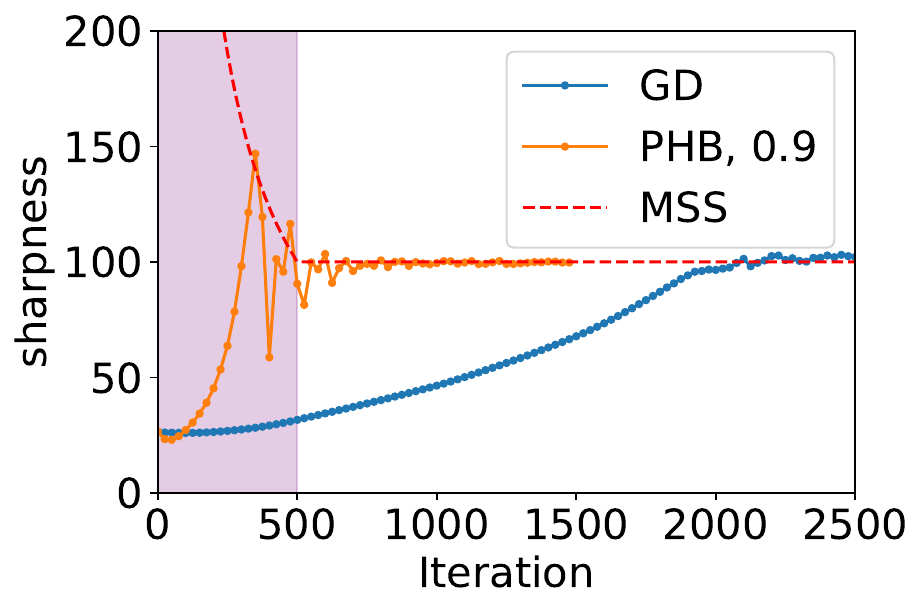}
    }
    \hfill
    \subfigure[Width-7000 FCN, $\eta_i=0.001$, $\eta_f=0.01$]{
        \includegraphics[width=0.29\textwidth]{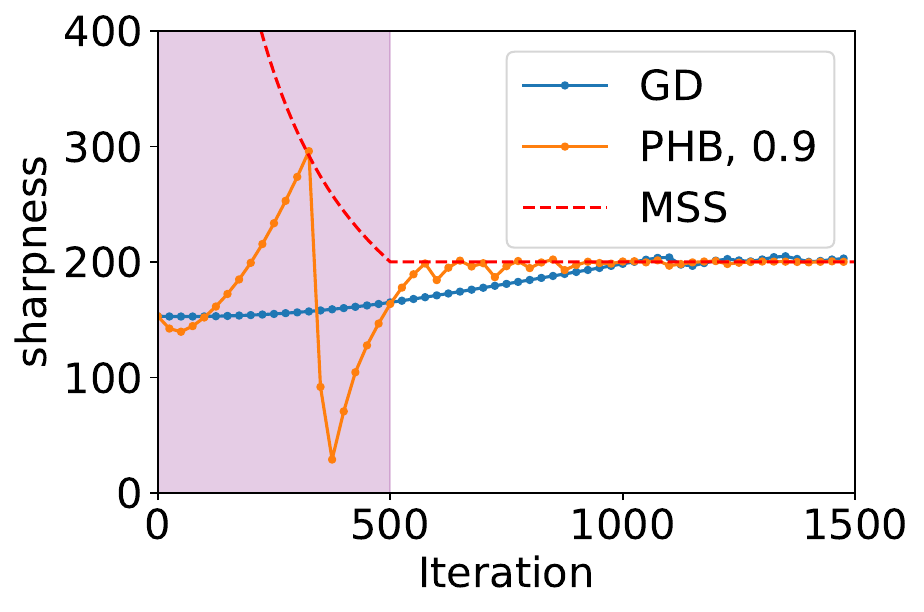}
    }
    \hfill
    \rulesep
    \subfigure[Width-200 FCN, $\eta_i=0.001$, $\eta_f=0.05$]{
    \includegraphics[width=0.31\textwidth]{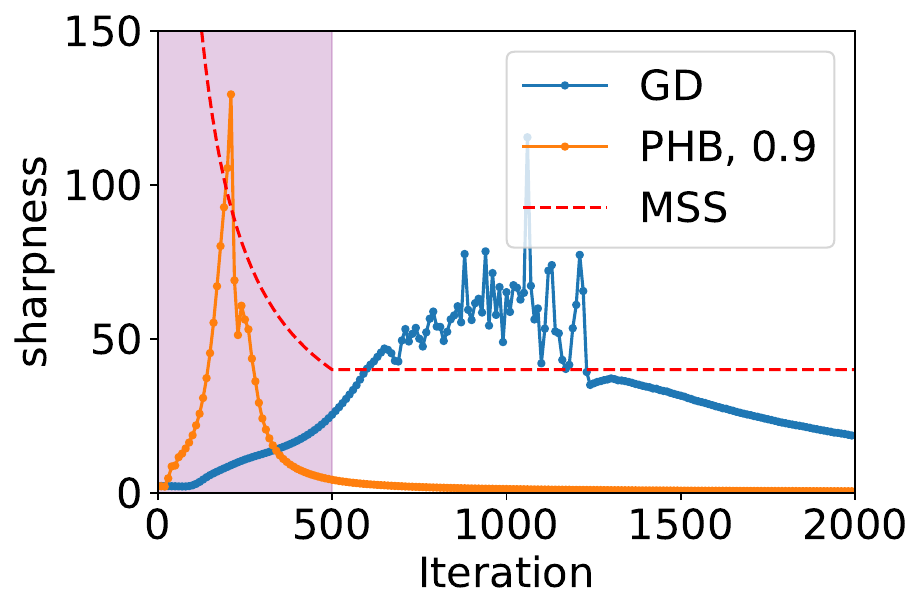}
    }
    \caption{Neural Networks trained on (a-c) 1k and (d-f) 5k subset of CIFAR10~\citep{krizhevsky2009learning}.
    For (a,b,d,e), we use the MSE loss, and for (c,f) we the CE loss.
    All FCNs are 3-layer and use ReLU activation. The shaded region is the linear warmup period. 
    }
    \label{fig:cifar10-exp}
\end{figure*}

\subsection{Nonlinear Neural Networks}
\label{sec:modern-arch}
% We perform additional experiments on a 3-layer fully connected ReLU network (FCN) with a synthetic dataset and a ResNet20 with a subset of the CIFAR10 dataset; full experimental details are deferred to Appendix~\ref{app:modern}.
To show that the phenomenon is not limited to LDNs, we also conduct experiments with narrow and wide fully-connected networks (FCN) on 1k and 5k-datapoint subsets of CIFAR10~\citep{krizhevsky2009learning} for both MSE {\it and} cross entropy (CE) losses.
The results are shown in Figure~\ref{fig:cifar10-exp}. It is clear that the catapults observed in PHB reduce the sharpness farther below the MSS than GD. It can also be observed that PHB is more ``sensitive'' to instability, in that sharpness being slightly above the MSS quickly causes catapults; this is in contrast to GD, where the sharpness sometimes hover above the MSS curve. % (2) the catapults observed in PHB dynamics reduces the sharpness farther below the MSS than GD.
%the large catapults occur across various considered settings.
The late-phase dynamics of CE experiments look different from MSE, because sharpness eventually decreases. This can be attributed to the fact that CE loss drives the iterates to flatter region once the model achieves 100\% training accuracy. Even in the case of CE, PHB displays a large drop of sharpness at the moment the sharpness curve touches the MSS curve.
% Notably, the sharpness of GD in the MSE experiments remains close to the MSS throughout training, going through the edge of stability~\citep{cohen2021stability} and highly unstable training. In contrast, the sharpness of PHB displays large catapults, resulting in the final sharpness stabilizing well below the MSS. 
Additional experiments, including results for ResNet20 and for $\beta = 0.99$, can be found in Appendices~\ref{app:modern} and \ref{app:099}.

\section{Why Large Catapult? Because Momentum Prolongs Self-Stabilization}
\label{sec:prolong}
In this section, we propose a hypothesis for the mechanism that momentum causes the large catapults.
To do that, we first review the \emph{self-stabilization}~\citep{damian2023stabilization} mechanism of GD. 
We use $\vw_{\max}(\bm\theta_\star)$ to denote the eigenvector corresponding to the sharpness $S$ (leading eigenvalue of the training loss Hessian) at some minimum $\bm\theta_\star$.
% First, recall that at each minimum $(u, 0)$, $\vw_{\max}$ is $(0, 1)$ with eigenvalue $S(u, 0) = u^2$. 
% First, we introduce the relevant notation. We define $\lambda_i$ as the $i^{th}$ largest eigenvalue of the training loss Hessian and $w_i$ as its corresponding eigenvector. Thus $\lambda_1$ is the sharpness, and $w_1$ is the eigenvector corresponding to the sharpness.
Self-stabilization consists of four stages: 

\textbf{Stage 1 (Progressive Sharpening (PS)\footnote{Although \citet{damian2023stabilization} assume that PS always occurs, it has been reported that it does not occur for simple settings~\citep{cohen2021stability,zhu2023catapult}.}).} The sharpness of the iterates increases to reach the MSS.
% In our toy model, we start from an already unstable initialization, so there is no Stage~1.

\textbf{Stage 2 (Blowup).} Once the sharpness becomes larger than the MSS, the iterates oscillate and diverge along $\vw_{\max}$.
In this stage, the loss starts to increase sharply, depending on the amount of divergence.
% In our toy model, the dynamics start from Stage~2 and the divergence along $\vw_{\max}$ corresponds to the oscillations along the $v$-direction.

\textbf{Stage 3 (Self-Stabilization).} {\it Simultaneously}, the divergence along $\vw_{\max}$ induces a drift along the direction of $-\nabla S$ which decreases the sharpness. This is due to the cubic term of the Taylor expansion of
$\nabla \gL: \nabla^3 \gL(\vw_{\max}(\bm\theta), 
\vw_{\max}(\bm\theta)) = \nabla S(\bm\theta)$ for any\footnote{Here, one needs to assume that the leading eigenvector of $\nabla^2 \gL(\bm\theta)$ is unique.} $\bm\theta$~\citep[Lemma 2]{damian2023stabilization}; GD iterates oscillating along $\vw_{\max}$ introduce components in the direction of $-\nabla S$ to GD updates.
% $\nabla \gL(\cdot): \nabla^3 \gL(w_{\max}(\theta), w_{\max}(\theta)) \langle w_{\max}(\theta_\star), \theta_t - \theta_\star \rangle^2 \approx \nabla S(\theta_\star) \langle w_{\max}(\theta_\star), \theta_t - \theta_\star \rangle^2$~\citep[Lemma 2]{damian2023stabilization}\footnote{Here, one needs to assume that the leading eigenvector of $\nabla^2 \gL(\bm\theta)$ is unique}.
 % elaborate using third-order~\citep[Lemma 1]{damian2023stabilization} $\nabla^3 \gL = \nabla S$

% In the toy model, the oscillations along the $v$-direction induce a movement in the negative $u$-direction.

\textbf{Stage 4 (Return to Stability).} When the sharpness drops below the MSS, the oscillation in the $\vw_{\max}$ direction dampens, and dynamics become stable again.
% This corresponds to the late-stage dynamics of our toy model where oscillations in the $v$-direction start to decay until convergence, during which the iterates continue to be pushed in the negative $u$-direction.

Although this self-stabilization was developed to explain the mechanism of the edge-of-stability (EoS) for gradient descent, it has been already pointed out that EoS can be regarded as ``a never-ending series of micro-catapults''~\citep[Section 3.2]{cohen2021stability}.
Thus, it is natural to see that 
\begin{center}
    {\it a catapult can be explained as a single round of self-stabilization (Stages~2--4).}
\end{center}

We remark that such an explicit connection between catapults and self-stabilization was absent from prior literature on catapults~\citep{zhu2023catapult,zhu2023catapultquadratic,lewkowycz2020catapult}, where the focus was on either analytically proving or empirically showing the existence of a loss spike and sharpness\footnote{In prior works, the notion of sharpness was replaced with maximum eigenvalue of the NTK matrix.} drop.
A recent paper by \cite{zhu2023catapult} shows that for SGD, the loss spikes in training with SGD are ``caused'' by the catapults in the top eigenspace of the NTK, but they did not provide any explanation on the root cause of {\it why/how} catapults occur.

% \subsection{Main Hypothesis: Momentum Prolongs Self-Stabilization}
% \label{subsec:hypotheses}
Given this correspondence between catapults and self-stabilization, it is natural to consider the role of momentum in this process.
Drawing from known intuitions for momentum dynamics~\citep{goh2017momentum,pedregosa2023momentum,muhlebach2021momentum}, we propose the following hypothesis:
\begin{hypothesis}
\label{hypothesis}
    Polyak's momentum prolongs self-stabilization in the following sense:
    \begin{enumerate}
        \item \label{hyp:stage2-3}As the iterates oscillate and diverge along $\vw_{\max}$, the momentum in the direction of $-\nabla S$ ``builds up'' (Stages 2--3).
        % momentum prolongs (and possibly amplifies) the divergence along $\vw_{\max}$, contributing more to the movement in the $-\nabla S$ direction (Stage 2--3).
        \item \label{hyp:stage4}Even after the sharpness drops below the MSS, momentum prolongs oscillation along $\vw_{\max}$, which in turn prolongs the movement in the $-\nabla S$ direction (Stage 4).
    \end{enumerate}
\end{hypothesis}

\section{Verifying our Hypothesis in Simple Networks}
\label{sec:toy}
In this section, we verify Hypothesis~\ref{hypothesis} in two simple networks: a ReLU scalar network and LDNs.
We first empirically verify our hypothesis and provide a theoretical characterization of the sharpness reduction for GD and PHB for the ReLU scalar network.
We then provide additional empirical evidence for our hypothesis in LDNs.
% In this section, we will verify our hypothesis theoretically and empirically through a toy example.
% {\color{red}Outline}
% To build some intuition and solidify our hypothesis regarding the mechanism behind the large catapults

\subsection{Simple Network \#1. ReLU Scalar Network}
The ReLU scalar network is defined as $f(x; u, v) = v \sigma(ux)$ where $\sigma: \sR \to \sR$ is the ReLU activation function, $\sigma(\cdot) = \max(\cdot, 0)$.
We consider the single data point case $(x,y)=(1,0)$ trained using MSE loss function, given by $\gL(u, v) = \tfrac{1}{2} u^2 v^2 \indicator[u \geq 0],$ where $\indicator[\cdot]$ is the indicator function.
The sharpness values at the minima $(u, 0)$ are $S(u, 0) = u^2$ for $u \geq 0$ and $S(u, v) = 0$ for $u < 0$.
% \begin{equation*}
%     S(u, 0) = u^2, \quad S(u, v) = 0, \ u \leq 0.
% \end{equation*}
% \begin{align*}
%     S(u,v) &=
%     \begin{cases}
%         u^2 & \text{ if } v = 0, \\
%         0 &\text{ if } u \le 0.
%     \end{cases}
% \end{align*}
We consider constant learning rate $(1 + \beta) \eta$ to simplify the analysis.\footnote{In Figure~\ref{fig:toy-model-warmup} of \Cref{app:warmup-necessity}, we show that catapults can also be observed in this model with linear warmup.}
% Note that due to our choice of initialization, we are On the other hand, in our theoretical analyses, we can focus choose a constant step size given a proper choice of initialization close to the minima with sharpness slightly above the MSS thus not requiring the use of warmup.
Recall from Remark~\ref{rmk:roleofwarmup} that catapults under linear warmup occur when the learning rate is just large enough so that the sharpness is slightly above the MSS. We reproduce this setup for constant learning rates by fixing the learning rate $\eta$ and initializing close to an unstable minimum: $(u_0,v_0)$ with $u_0^2 = \frac{2 + \epsilon}{\eta}$ for some small $\epsilon, |v_0| \in (0, 1)$.

%Although this paper mainly considers PHB with linear warmup, our experiment with step warmup shows that constant (but large) $\eta$ can induce catapults if the parameters are initialized close to a minimum.}
% The divergent behavior occurs when $\eta > 2/u_0^2$.
% Via some tedious computations, one can see that for $u \geq 0$,
% \begin{equation}
%     \lambda_{\max}(u, v) := \lambda_{\max} \left(\nabla^2 \gL(u, v)\right) = u^2 \frac{1 + \left( \frac{v}{u} \right)^2 + \sqrt{\left( 1 + \left( \frac{v}{u} \right)^2 \right)^2 + 12 \left( \frac{v}{u} \right)^2}}{2}.
% \end{equation}
% When $v = 0$ and $u \geq 0$, the sharpness is $\lambda_{\max}(u, 0) = u^2$.
% Especially note that $\lambda_{\max}(u, 0) = u^2$, and if $u < 0$, then the loss, gradient, and sharpness are all zero.

\begin{figure}[!t]
\centering
\includegraphics[width=0.8\linewidth]{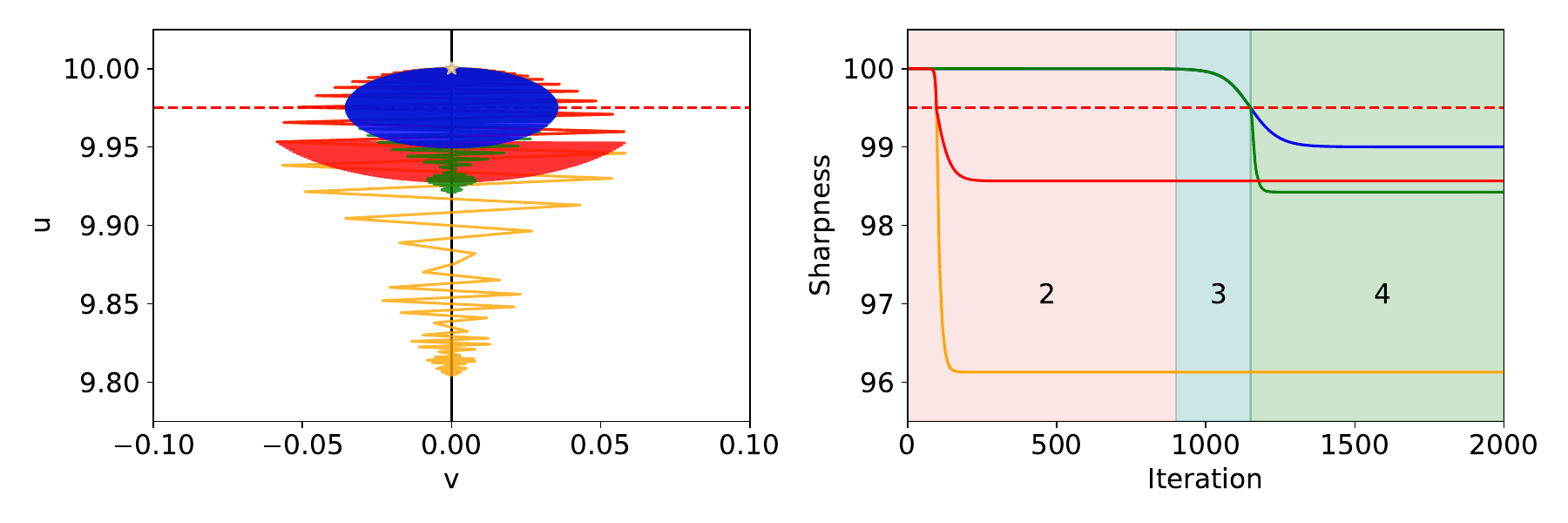} %0.45
\caption{(Left) Trajectories of {\color{blue} GD}, {\color{orange} PHB}, {\color{darkgreen} GD $\to$ PHB}, {\color{red} PHB $\to$ GD} with $\beta=0.9$, $\eta = (2+\epsilon)/u_0^2$ where $\epsilon=0.01$, $(u_0, v_0)=(10, 10^{-6})$, and no warmup. (Right) The self-stabilization stages for {\color{blue}GD} are highlighted and labeled in the sharpness plot. The MSS is shown as the red dotted line.
}
\label{fig:scalar-net-toy}
\end{figure}
% \begin{figure}
%     \centering
%     \subfigure[$u_\infty$ bounds for the toy example. The derived bound (blue) and $u_\infty$ (red dashed line) overlap, showing that the bound is tight.]{
%         \includegraphics[width=0.32\textwidth]{figures/scalar-net/scalar-loss-Cu-Cv-beta.pdf}
%     }
%     \hfill
%     \subfigure[Experiments switching momentum on/off]{
%         \includegraphics[width=0.64\textwidth]{figures/scalar-net/scalar-net-switch1.pdf}
%         \label{subfig:toy-momentum-on-off}
%     }
%     \caption{Toy example experiments}
%     \label{fig:ub-bound}
% \end{figure}

\subsubsection{Qualitative Observations}
In Figure~\ref{fig:scalar-net-toy}, we provide the trajectory and sharpness plots of GD and PHB for our toy model to illustrate the self-stabilization and our hypothesis.
First, we note that we do not observe any PS due to the model's simplicity. Since the initialization is already unstable, both trajectories start from Stage 2 (Blowup) in the language of self-stabilization.
% We denote the GD/PHB iterates initialized at $(u_0, v_0)$ as $(u_t, v_t)$.

\paragraph{GD Dynamics.}
% Given a sufficiently large learning rate $\eta > 2/u_0^2$, the dynamics are initially divergent.
% As long as $u_t^2 > 2/\eta$, $v_t$ alternates signs and $|v_t|$ increases monotonically. This regime occurs when the iterates are above the red dashed line depicting the $u$-value corresponding to the MSS. As $|v_t|$ monotonically increases, $1 - \eta v_t^2$ decreases, thus $u_t$ monotonically decreases at an increasingly fast rate.
As the initial sharpness is already above the MSS, the dynamics go through divergent oscillations in the $v$-direction (Stage 2).
This, in turn, amplifies the movement in the negative $u$-direction, reducing the sharpness and stabilizing the dynamics (Stage 3).
Eventually, the sharpness decreases below the MSS, and the oscillations in the $v$-direction are dampened until the iterates converge to a minimum of sharpness {\it just below the MSS} (Stage 4).
This is illustrated in Figure~\ref{fig:scalar-net-toy} with color {\color{blue} blue}, where one can actually see an approximate symmetry of the {\color{blue} GD} dynamics around $(\sqrt{2/\eta}, 0)$ (dashed {\color{red}red} line).
This interplay between oscillation and stabilization is precisely the self-stabilization for GD~\citep{damian2023stabilization}.

% {\color{red} {\bf OPTIONAL; delete?} To see this more clearly, we consider the GD update equation of our model.
% At first, due to our initialization, we have that $1 - \eta u_t^2 \approx -1 - \varepsilon < -1$, causing an increasing oscillation in the $v$-direction (blowup).
% At the same time, as $0 < 1 - \eta v_t^2 < 1$, $u_t$ slowly decreases (unnoticeably).
% When $|v_t|$ is sufficiently large, $u_t$ decreases more aggressively (self-stabilization).
% Once $u_t < \sqrt{\frac{2}{\eta}}$, as $-1 < 1 - \eta u_t^2$, oscillation in the $v$-direction starts to dampen until convergence, and even during the dampening, $u_t$ continues to decrease (at a slower rate as $t$ progresses).}

\paragraph{PHB Dynamics.} 
The dynamics of PHB are initially similar to GD dynamics when the sharpness is above the MSS. However, momentum further accelerates the negative $u$-direction in Stages 2--3. This momentum persists even after the sharpness is below the MSS (Stage 4), allowing momentum to \emph{prolong} the oscillation along $v$ direction, at the same time driving the iterates to flatter minima.
This is illustrated in Figure~\ref{fig:scalar-net-toy} with color {\color{orange} orange}, where one can see that momentum does not simply accelerate the GD dynamics but shows a \emph{qualitatively} different behavior, breaking the symmetry around $(\sqrt{2/\eta}, 0)$ that was present for {\color{blue} GD}.
We also note that although we rescale $\eta$ to $\eta (1 + \beta)$ for {\color{orange} PHB}, the sharpness reduction and the distance traveled in the $u$-direction of PHB is about 3.89 and 3.915 times that of {\color{blue} GD}, respectively.
Both are by a factor more than $1 + \beta = 1.9$, which confirms that the increased sharpness drop is not simply the result of the rescaled learning rate.

\paragraph{Further Verifying Hypothesis~\ref{hypothesis}.}
To further empirically validate our Hypothesis~\ref{hypothesis}, we consider two additional settings: {\color{darkgreen} GD $\to$ PHB} and {\color{red} PHB $\to$ GD}, where $A \to B$ means that we first run $A$, then switch to $B$ (with learning rate rescaling) once the sharpness crosses the MSS.
% In {\color{blue} (1)} and {\color{orange} (2)}, we fully train the toy model using {\color{blue} GD} and {\color{orange} PHB} (respectively) as usual. For {\color{darkgreen} (3)}, we first train using GD up until the sharpness crosses the MSS and then switch to training using PHB (adjusting $\eta$ as needed to match the MSS). Similarly, for {\color{red} (4)}, we begin training using PHB and switch to GD after the sharpness crosses the MSS (and adjusting $\eta$ as needed).
If Hypothesis~\ref{hypothesis}.\ref{hyp:stage2-3} holds, then due to the effect of momentum before crossing the MSS, one would expect {\color{red} PHB $\to$ GD} to experience a larger sharpness reduction than {\color{blue} GD}.
Similarly, if Hypothesis~\ref{hypothesis}.\ref{hyp:stage4} holds, then one would expect {\color{darkgreen} GD $\to$ PHB} to experience a larger sharpness reduction than {\color{blue} GD}.
Indeed, as shown in Figure~\ref{fig:scalar-net-toy}, in increasing order of sharpness reduction, we have {\color{blue} GD} < {\color{red} PHB $\to$ GD} < {\color{darkgreen} GD $\to$ PHB} < {\color{orange} PHB}. Indeed, this shows that although the effect of each part of our hypothesis (Stages 2--3 and Stage 4) is already sufficient to display a larger sharpness reduction, the \emph{combination} of these two factors results in an even larger sharpness reduction for {\color{orange} PHB}.
% Additionally, one can see that momentum has a greater effect in Stage 4 than it does in Stages 2-3 since {\color{darkgreen} GD $\to$ PHB} undergoes a larger sharpness reduction than {\color{red} PHB $\to$ GD}.

\subsubsection{Theoretical Analysis}
The GD/PHB dynamics of the toy example with learning rate $\eta > 0$ and momentum parameter $\beta \in [0, 1)$ are given as follows:
\begin{align*}
    u_{t+1} &= \left( 1 - \eta (1 + \beta) v_t^2 \indicator[u_t \geq 0] \right) u_t + \beta(u_t - u_{t-1}) \\
    v_{t+1} &= \left( 1 - \eta (1 + \beta) u_t^2 \indicator[u_t \geq 0] \right) v_t + \beta(v_t - v_{t-1}).
\end{align*}
Note that the learning rate is rescaled by a factor of $1 + \beta$ to keep the MSS constant at $\sqrt{\frac{2}{\eta}}$.
The proofs for all the statements provided here are deferred to Appendix~\ref{app:proofs}.
% Note that the coefficient $(1 - \eta u_t^2)$ in the update rule for $v_t$ is negative when $u_t^2 > 1/\eta$. Also, $|1 - \eta u_t^2| > 1$ when $u_t^2 > 2/\eta$ which corresponds to the regime where the sharpness is larger than the MSS.

% The following assumption constrains the setting ($\beta, \eta, \epsilon, v_0$) such that the GD/PHB dynamics don't suddenly end up in the $u_t < 0$ region, where the gradient completely vanishes.
% \begin{assumption}
% \label{assumption:positive}
%     $\inf_{t \geq 0} u_t \geq 0$.
% \end{assumption}
The following lemma states that for our dynamics, $\{u_t\}$ is monotone decreasing and convergent:
\begin{lemma}
\label{lem:u}
    $u_\infty := \lim_{t \rightarrow \infty} u_t \leq u_{t+1} \leq u_t$ for all $t \geq 0$.
    Furthermore, if $\tau_0 := \inf\left\{ t \geq 0 : u_t < 0 \right\} < \infty$, then we have that $u_\infty = \frac{u_{\tau_0} - \beta u_{\tau_0 - 1}}{1 - \beta}$.
\end{lemma}
Notice that our experiments correspond to the $\tau_0 = \infty$. So the question remains: how to characterize $u_\infty$ when $\tau_0 = \infty$, which we will assume from hereon and forth.

For that, we first introduce the following quantities:
\begin{equation}
\label{eq:CuCv}
    \tau_u := \inf\left\{ t \geq 0 : u_t^2 < \frac{2 - \epsilon}{\eta} \right\}, \quad
    C_u := \frac{u_{\tau_u} - \beta u_{\tau_u - 1}}{1-\beta},\quad
    C_v := \frac{1+\beta}{1-\beta}\sum_{t = \tau_u}^\infty v_t^2.
\end{equation}
They all are deterministic functions of ($\beta, \eta, \epsilon, v_0$), as GD/PHB are deterministic.
For our analysis, we fix some $\beta, \eta, \epsilon, v_0$, with appropriate conditions that will be elaborated in the statements.
We will drop the dependency from hereon and forth for notational simplicity.

The following theorem characterizes an upper bound on $u_\infty$:
\begin{theorem}
\label{thm:toy}
Suppose that $\tau_0 = \infty$ and $\epsilon, |v_0| < 1$.
Then, we have that
\begin{equation}
\label{eq:uinfub}
    u_\infty = \lim_{t \rightarrow \infty} u_t \leq 
    \frac{C_u}{1 + \eta C_v}
    =: 
    \overline{u}_\infty ,
    %\leq \frac{1}{1 + \eta \frac{1 + \beta}{1 - \beta} C_v} \sqrt{\frac{2 - \epsilon}{\eta}},
\end{equation}
which implies a lower bound on the ``asymptotic'' sharpness reduction $\Delta S_\infty := S(u_0, 0) - S(u_\infty, 0) = u_0^2 - u_\infty^2$.
\end{theorem}

Several observations can be made here.
According to Theorem~\ref{thm:toy}, the RHS decreases (i.e., larger displacement in the $-u$ direction) with larger $C_v$ and smaller $C_u$. 
First, notice that $C_u$ can be rewritten as $C_u = u_{\tau_u} + \frac{\beta}{1-\beta} (u_{\tau_u} - u_{\tau_u-1})$. Considering that $u_{\tau_u} \approx \sqrt{\frac{2-\epsilon}{\eta}}$, the key to understanding the behavior of $C_u$ is the second term, which can be viewed as the ``momentum'' that has built up in the $+\nabla S$ direction \emph{during} Stages 2--3. As per Hypothesis~\ref{hypothesis}.\ref{hyp:stage2-3}, it is natural to expect that $C_u$ will decrease for larger values of $\beta$ (recall that $u_{\tau_u} - u_{\tau_u-1} < 0$).
On the other hand, $C_v$ accumulates when there is more oscillation in the $v$ direction. Connecting this to Hypothesis~\ref{hypothesis}.\ref{hyp:stage4}, $C_v$ measures the movement in the $-\nabla S$ direction caused by divergence along $w_{\max}$ {\it after} the iterates cross the MSS (Stage 4), with the scaling factor $\frac{1+\beta}{1-\beta}$ capturing the ``prolonging'' effects of momentum.
%Thus, Theorem~\ref{thm:toy} provides a rigorous proof of our Hypothesis~\ref{hypothesis}.

% First, the RHS decreases with stronger momentum ($\beta \uparrow 1$), greater movement in the $v$-direction ($C_v \uparrow$), or larger movement at the time the iterates cross the MSS ($C_u \downarrow$), all of which match well with our intuition. The effect of smaller $C_u$ aligns well with Hypothesis~\ref{hypothesis}.\ref{hyp:stage2-3} since $C_u$ decreases as the accumulated momentum in the direction of $-\nabla S$ til the iterates cross the MSS (Stages 2--3) increases.
% On the other hand, $C_v$ measures the movement in the $-\nabla S$ direction caused by divergence along $w_{\max}$ {\it after} the iterates cross the MSS (Stage 4), with the scaling factor $\frac{1+\beta}{1-\beta}$ capturing the ``prolonging'' effects of momentum in Hypothesis~\ref{hypothesis}.\ref{hyp:stage4}.

% Importantly, the $C_v$ measures the total divergence along the unstable direction in the self-stabilization (Stage 4), and thus, \emph{Theorem~\ref{thm:toy} provides a rigorous proof of our Hypothesis~\ref{hypothesis}.\ref{hyp:stage4}} {\color{red} INCLUDE $C_u$ and rewrite discussion on $C_v$ and $C_u$. $\eta C_v \frac{1+\beta}{1-\beta}$ captures the effect of oscillations on movement in the $-\nabla S$ direction: $\eta (1+\beta) C_v$ is the update in the $u$-direction, $\frac{1}{1-\beta}$ captures the effect of momentum.}

\begin{figure*}[!t]
    \centering
    \subfigure[Change of of $C_u$ and $\frac{1}{1 + \eta C_v}$]{
        \includegraphics[width=0.4\textwidth]{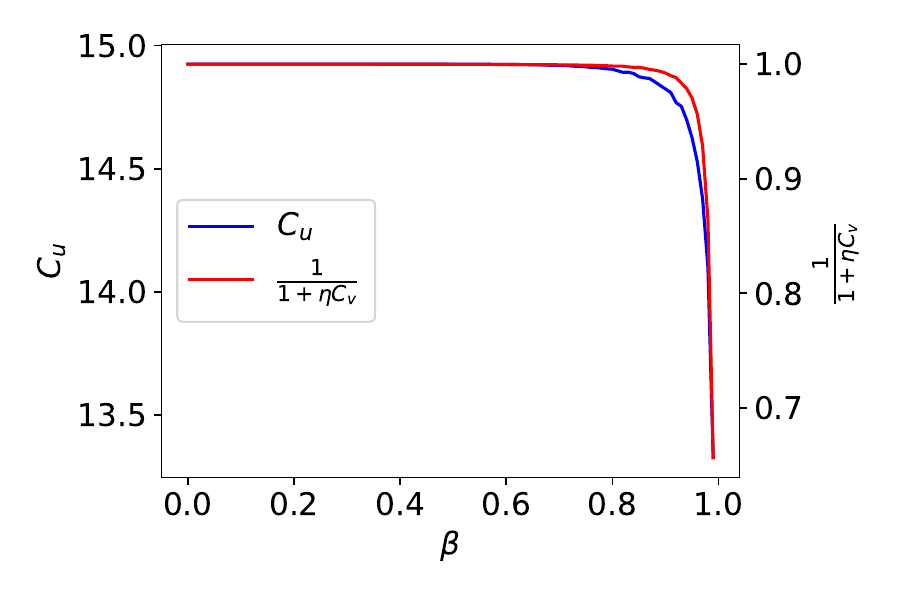}
        \label{subfig:scalar-Cu-Cv-bound}
    }
    \hfill
    \subfigure[Comparing sharpness reduction]{
        \includegraphics[width=0.4\textwidth]{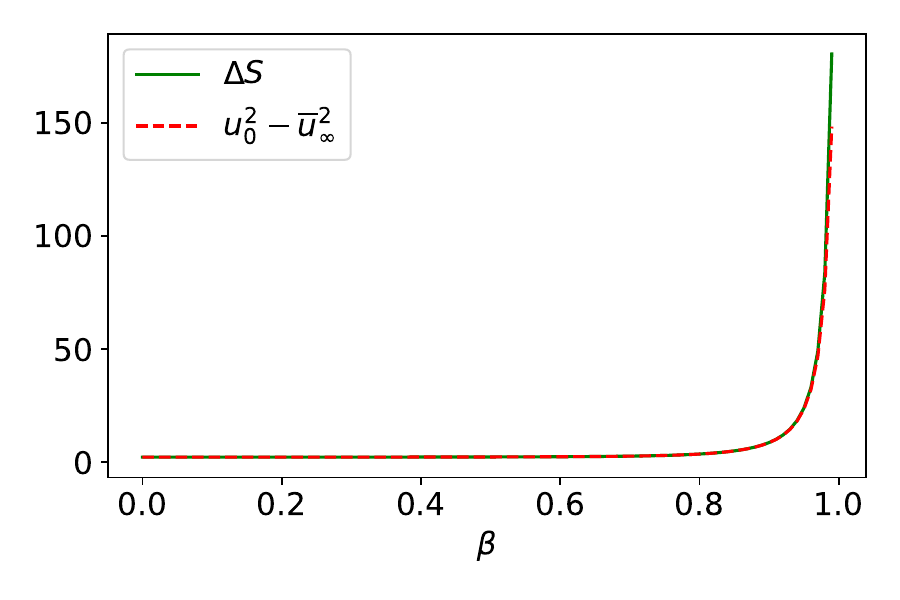}
        \label{subfig:scalar-sharpness-reduction}
    }
    \caption{Numerical verification of Theorem~\ref{thm:toy}.}
    \label{fig:ub-bound}
\end{figure*}

% \begin{wrapfigure}[24]{r}[1pt]{0.4\textwidth}
%   \centering
%     \includegraphics[width=0.4\textwidth]{figures/scalar-net/scalar-loss-sharpness-reduction-beta.pdf}
%     \includegraphics[width=0.4\textwidth]{figures/scalar-net/scalar-loss-Cu-Cv-beta.pdf}
%   \caption{\label{fig:ub-bound}Numerical verification of Theorem~\ref{thm:toy}.
%   % $u_\infty$ bounds for the toy example. The derived bound (blue) and $u_\infty$ (red dashed line) overlap, showing that 0the bound is tight.
%   }
% \end{wrapfigure}
We show numerically that Theorem~\ref{thm:toy} is tight by rerunning the scalar ReLU network experiment with PHB for varying $\beta \in \{0.0, 0.01, \cdots, 0.99\}$ for $T = 10^5$ iterations.
In Figure~\ref{subfig:scalar-Cu-Cv-bound}, we plot $C_u$ and $\frac{1}{1 + \eta C_v}$ across $\beta$. 
We expect that both terms will decrease with increasing $\beta$, which turns out to be true; this confirms Hypothesis~\ref{hypothesis} for the scalar ReLU network case.
%, we expect the sharpness reduction to increase as both quantities decrease, which is the case. 
Connecting this observation to our previous experiment in Figure~\ref{fig:scalar-net-toy}, {\color{red} PHB $\to$ GD} corresponds to smaller $C_u$ due to the effects of momentum in stages 2-3 while {\color{darkgreen} GD $\to$ PHB} corresponds to larger $C_v$ (hence a smaller $\frac{1}{1 + \eta C_v}$) due to momentum prolonging the oscillations in stage 4. Each of these terms explains the increased sharpness reduction in {\color{red} PHB $\to$ GD} and {\color{darkgreen} GD $\to$ PHB} when compared to {\color{blue} GD}. The combined effect of these two quantities is shown in Figure~\ref{subfig:scalar-sharpness-reduction} where we plot the theoretical sharpness reduction $u_0^2 - \overline{u}_\infty^2$ and $\Delta S = u_0^2 - u_T^2$.
Note that $u_0^2 - \overline{u}_\infty^2$ is a tight lower bound on $\Delta S$.

As a counterpart, in the following theorem, we provide an asymptotic (sufficiently small $\epsilon$ and $v_0^2$, and moderate learning rate $\eta$) lower bound on $u_\infty$ for {\it GD}:
\begin{theorem}[Informal]
\label{thm:toy-GD}
    For GD ($\beta = 0$), suppose that $\epsilon = o(1)$, $v_0^2 = \gO(\epsilon)$, and $\eta = \Theta(1)$.
    Then, we have that 
    $u_\infty \geq \sqrt{\frac{2}{\eta}} - \gO( \sqrt{\epsilon} ),$ which implies $\Delta S_\infty \leq \gO(\sqrt{\epsilon}).$
\end{theorem}
\begin{proof}[Proof sketch.]
    Inspired by the empirical observations that the GD iterates' envelope resembles an ellipse ({\color{blue} blue} trajectory in Figure~\ref{fig:scalar-net-toy}), we show that the elliptical ``energy'' $P_t := \left( u_t - \sqrt{\frac{2}{\eta}} \right)^2 + \frac{1}{2}v_t^2$ is approximately conserved over $t$ and $P_\infty = \gO(\epsilon)$, which implies the desired bound on $u_\infty$.
\end{proof}

In Theorem~\ref{thm:toy}, for GD ($\beta = 0$), we have that $u_\infty \leq \frac{1}{1 + \gO(1)} \sqrt{\frac{2 - \epsilon}{\eta}}$.
This follows from the fact that $C_v = \gO(1)$ for GD, which we show in the process of proving Theorem~\ref{thm:toy-GD}.
This then implies that $\Delta S_\infty \geq \Omega(\epsilon)$, off by a factor of $\sqrt{\epsilon}$ compared to Theorem~\ref{thm:toy-GD}.
% Comparing this to Theorem~\ref{thm:toy-GD}, we have tight lower and upper asymptotic bounds for $\Delta S$.
We leave extending Theorem~\ref{thm:toy-GD} to PHB for future work. The key difficulty is that the energy argument fails; this can be seen empirically in the PHB trajectory ({\color{orange} orange}) in Figure~\ref{fig:scalar-net-toy}.

% \begin{figure}
%     \centering
%     \subfigure[$u_\infty$ bounds for the toy example. The derived bound (blue) and $u_\infty$ (red dashed line) overlap, showing that the bound is tight.]{
%         \includegraphics[width=0.32\textwidth]{figures/scalar-net/scalar-loss-Cu-Cv-beta.pdf}
%     }
%     \hfill
%     \subfigure[Experiments switching momentum on/off]{
%         \includegraphics[width=0.64\textwidth]{figures/scalar-net/scalar-net-switch1.pdf}
%         \label{subfig:toy-momentum-on-off}
%     }
%     \caption{Toy example experiments}
%     \label{fig:ub-bound}
% \end{figure}

\paragraph{Comparisons to Prior Works.}
Our model is essentially identical to the 1D $uv$-model analyzed in \citet{kalra2023catapult,lewkowycz2020catapult}.
However, neither work considers the {\it PHB} dynamics in the parameter space $(u_t, v_t)$. Furthermore, \cite{lewkowycz2020catapult} additionally require NTK scaling with a finite but sufficiently large width.
Our characterization of the GD dynamics share many characteristics with the single-neuron neural network described and analyzed by \citet{ahn2022threshold}, like the quasi-static principle (an ellipse-like envelope) and final resulting sharpness being close to the MSS.
However, their analysis cannot be directly applied to our scenario, as $\ell(u) = \frac{1}{2} u^2$ does not satisfy their assumptions ($\ell$ is globally Lipschitz and $\ell'(u) / u$ decays locally away from $u = 0$).
Although we focus on a simple model, our Theorems~\ref{thm:toy} and \ref{thm:toy-GD} provide a rigorous characterization of the parameter dynamics of GD/PHB with large learning rates beyond the previously considered assumptions.

\subsection{Simple Network \#2. Linear Diagonal Networks}
\label{sec:ldn-hypothesis}

% \begin{hypothesis}
    % For nonlinear loss function $\gL(\cdot)$, consider running GD/PHB from initialization $\bm\theta_0$ satisfying $\lambda_{\max}(\nabla^2 \gL(\bm\theta_0)) = \frac{2 + \epsilon}{\eta}$ for some small $\epsilon \in (0, 1)$.
    % Then, the $C_u$ and $C_v$ in our ReLU scalar network ``extends'' to the following:
    % \begin{equation}
    %     C_u := \left\langle \frac{\nabla S(\bm\theta_{\tau_u}^*)}{\bignorm{\nabla S(\bm\theta_{\tau_u}^*)}_2}, \bm\theta_{\tau_u}^* \right\rangle - \beta \left\langle \frac{\nabla S(\bm\theta_{\tau_u - 1}^*)}{\bignorm{\nabla S(\bm\theta_{\tau_u - 1}^*)}_2}, \bm\theta_{\tau_u - 1}^* \right\rangle, \
    %     C_v := \sum_{t= \tau_u}^\infty \left\langle \vw_{\max}(\bm\theta_t^*), \bm\theta_t - \bm\theta_t^* \right\rangle^2,
    % \end{equation}
    % where $\{\bm\theta_t\}_t$ is the trajectory of GD/PHB, $\tau_u := \inf\left\{ t \geq 0 : \lambda_{\max}(\bm\theta_t) < \frac{2 - \epsilon}{\eta} \right\}$, $\bm\theta_t^* := \argmin_{\bm\theta \in \sR^d : \gL(\bm\theta) = 0} \bignorm{\bm\theta - \bm\theta_t}_2$ is the projection of $\bm\theta_t$ onto the nearest manifold of minima, and $\vw_{\max}(\bm\theta^*)$ is the leading eigenvector\footnote{{\color{red} If this is non-unique, then we sum the squared inner product over all the leading eigenvectors as well.}} of $\nabla^2 \gL(\bm\theta^*)$.
% \end{hypothesis}

To show the potential of our verification of hypothesis extending beyond the scalar ReLU network, we revisit the LDNs and perform similar experiments as in the ReLU scalar networks, i.e., we plot the sharpness across 4 scenarios: {\color{blue} GD}, {\color{orange} PHB}, {\color{darkgreen} GD $\to$ PHB}, and {\color{red} PHB $\to$ GD}.
This is to show empirical evidence that Hypothesis~\ref{hypothesis} also holds for LDNs.
We first initialize the weights close to a minimum by running GD until the (MSE) loss is less than $0.001$.
We then run each scenario from that same initialization using $\eta = \frac{(2+\epsilon)(1+\beta)}{S_0}$ (and adjusting the learning rates as needed after the sharpness crosses the MSS) where $\epsilon = 0.03$ and $S_0$ is the sharpness of the initialization. As shown in Figure~\ref{subfig:LDNonoff}, in increasing order of reduction, we again have {\color{blue} GD} < {\color{red} PHB $\to$ GD} < {\color{darkgreen} GD $\to$ PHB} < {\color{orange} PHB}.

\begin{figure}
    \centering
    \subfigure[Sharpness of LDN for {\color{blue} (1) GD}, {\color{orange} (2) PHB}, {\color{red} (3) PHB $\to$ GD}, and {\color{darkgreen} (4) GD $\to$ PHB}. ]{
        \includegraphics[width=0.42\textwidth]{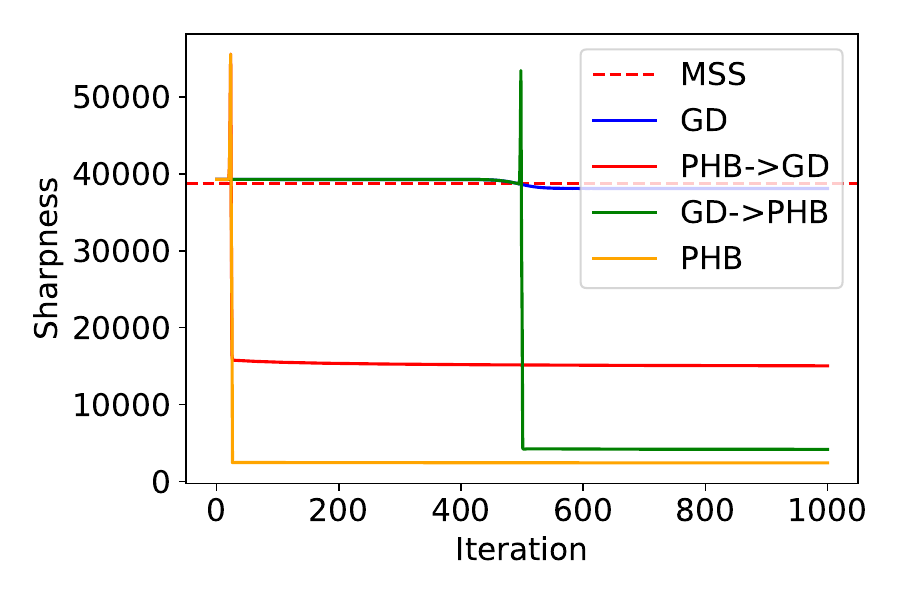}
        \label{subfig:LDNonoff}
    }
    \hfill
    \subfigure[Numerical verification of Theorem~\ref{thm:toy} for $\gL(u,v) = \frac{1}{2} (u^2 - v^2 - 1)^2$.]{
        \includegraphics[width=0.42\textwidth]{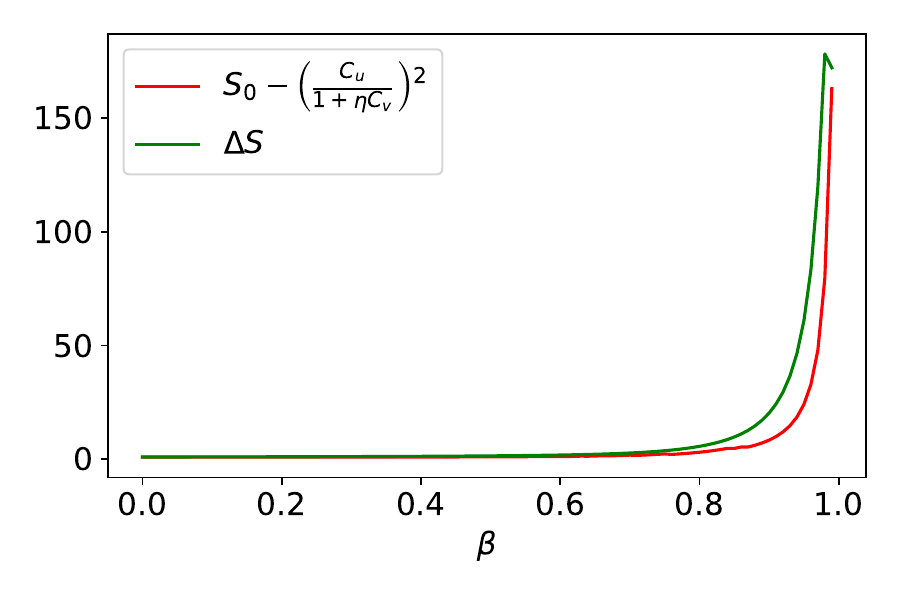}
        \label{subfig:LDNlb}
    }
    \caption{Empirical Validation of Hypothesis~\ref{hypothesis} and theory using LDNs}
    \label{fig:ldn-momentum-on-off}
\end{figure}

We now turn to the question of whether Theorem~\ref{thm:toy} can be extended to LDNs as well.
For a general loss function $\gL(\bm\theta)$, consider running GD and PHB (with rescaled learning rate) from initialization $\bm\theta_0$ satisfying $\lambda_{\max}(\nabla^2 \gL(\bm\theta_0)) = \frac{2 + \epsilon}{\eta}$ for some small $\epsilon \in (0, 1)$.
Then, we extend the definition of $C_u$ and $C_v$ (Eqn.~\eqref{eq:CuCv}) as the following. For the GD/PHB iterates $\{\bm\theta_t\}_{t\geq 0}$, we solve gradient flow (GF) starting from $\bm\theta_t$ to convergence, and obtain the solution $\bm\theta_t^*$, which can be viewed as the global minimum ``closest'' to $\bm\theta_t$. Then, for $\tau_u := \inf\left\{ t \geq 0 : \lambda_{\max}(\bm\theta_t^*) < \frac{2 - \epsilon}{\eta} \right\}$, define
\begin{equation}
\label{eq:CuCv2}
    %\tau_u := \inf\left\{ t \geq 0 : \lambda_{\max}(\bm\theta_t^*) < \frac{2 - \epsilon}{\eta} \right\},\quad
    C_u := \tfrac{\sqrt{S(\bm\theta_{\tau_u}^*)} - \beta \sqrt{S(\bm\theta_{\tau_u-1}^*)}}{1-\beta},\quad
    C_v := \tfrac{1+\beta}{1-\beta} \sum\nolimits_{t= \tau_u}^\infty \left\langle \vw_{\max}(\bm\theta_t^*), \bm\theta_t - \bm\theta_t^* \right\rangle^2,
\end{equation}
where $\vw_{\max}(\bm\theta^*)$ is the leading eigenvector
% \footnote{{\color{red} If this is non-unique, then we sum the squared inner product over all the leading eigenvectors as well.}} 
of $\nabla^2 \gL(\bm\theta^*)$.
Once we calculate Eqn.~\eqref{eq:CuCv2}, a ``lower bound'' on the sharpness displacement can be derived from Eqn.~\eqref{eq:uinfub}, although there is no theoretical guarantee that this should indeed hold for the new $\gL$. Still, we can numerically calculate it and assess the possibility of extending Theorem~\ref{thm:toy} to general functions.
% and derive reasonable lower bounds on the amount of sharpness reduction.

However, running GF to convergence at every iteration is the biggest bottleneck of numerical verification.
As a proof of concept, we consider a simple function motivated by the LDN architecture: $\gL(u,v) = \frac{1}{2} (u^2 - v^2 - 1)^2$.
This corresponds to the MSE loss of LDN for a single data point $(\vx,y) = (\ve_1, 1)$, and has the nice property that the GF trajectory $(u(t),v(t))$ satisfies $u(t)v(t) = u(0)v(0)$ for all $t \geq 0$ (see Appendix~\ref{sec:LDNLBdetails} for the proof). Using this property, we can calculate $\bm\theta_t^*$ for any $\bm\theta_t$, hence calculating the desired lower bound. Starting from a $\bm\theta_0 = (u_0, v_0)$ with sharpness $\frac{2+\epsilon}{\eta}$, we ran GD/PHB for a range of $\beta$'s, calculated the actual sharpness displacement and its lower bound. Appendix~\ref{sec:LDNLBdetails} contains further details on the experiment.
The results are shown in Figure~\ref{subfig:LDNlb}, which shows a similar trend as in the scalar ReLU network.
There is a small decrease in sharpness reduction at $\beta=0.99$, which can be attributed to overshooting; refer to Appendix~\ref{app:overshooting} for a discussion.
\vspace*{-3pt}
\section{Conclusion and Future Work}
\vspace*{-3pt}
In this paper, for the first time, we show that PHB with large learning rate induces large catapults, resulting in a much larger sharpness reduction than that of GD.
We first provide empirical evidence for this on linear diagonal networks (LDNs) and nonlinear neural networks.
We then hypothesize that the large catapult of PHB is caused by momentum {\it prolonging} self-stabilization~\citep{damian2023stabilization}.
We verify our hypothesis for ReLU scalar networks and LDNs and rigorously prove that it holds.

% as well as a scalar two-layer ReLU network, and offer an intuition of why such \emph{large catapults} occur from a self-stabilization perspective~\cite{damian2023stabilization}: momentum prolongs the sharpness reduction.

% Lastly, we empirically validate this phenomenon in 
\noindent\textbf{Future Work.}~~
Firstly for PHB, we sometimes observe a lack of PS after the large catapults, as shown in Figure~\ref{fig:cifar10-exp}(a--c). However, with larger datasets, PS appears to occur again, as shown in Figure~\ref{fig:cifar10-exp}(d--f). We conjecture that the lack of PS in ReLU networks is due to catapults aggressively reducing the number of active ReLU neurons~\citep{lu2020dying,andriushchenko2023sam}; refer to Appendix~\ref{app:overshooting} for a detailed discussion. Understanding the presence and absence of PS in different setup is an interesting direction.
% A more precise explanation of this phenomenon is left for future works.
% Firstly, for PHB, we sometimes observe a lack of PS after the large catapults, as shown in Figure~\ref{fig:cifar10-exp} (a-c). However, with larger datasets, PS appears to occur again, as shown in Figure~\ref{fig:cifar10-exp}(d-f). In general, the presence or absence of PS after a large catapult seems to depend on the dataset's size and the model architecture. One possible explanation for the lack of PS in {\it ReLU networks} is that the large catapults in PHB aggressively reduce the number of active ReLU neurons~\citep{lu2020dying,andriushchenko2023sam}; see Figure~\ref{fig:deadrelu} of Appendix~\ref{app:dead}. A more precise explanation of this phenomenon is left for future works.
Secondly, although PHB drives iterates towards flatter minima via the large catapults, we do not always observe better generalization.
% this is in contrast to the claims of \cite{zhu2023catapult,zhu2023catapultquadratic}, where the authors claim that catapults lead to better AGOP~\citep{radhakrishnan2022agop}, and thus better generalization.
Further exploring the connection between catapults and generalization is an interesting area for future research and may contribute to the growing literature on the connection between flatness and generalization~\citep{hochreiter1997flat,dinh2017sharp,even2023diagonal,andriushchenko2023sharpness}.

% The hypothesis that flat minima generalize better~\citep{hochreiter1997flat} has been actively studied, with ``counterexamples'' being recently proposed~\citep{dinh2017sharp,even2023diagonal,andriushchenko2023sharpness}. However, no solid conclusion has been reached, and exploring the connection between catapults and generalization would be an interesting area to explore. 

% \begin{remark}[Absence of Progressive Sharpening? - how to bridge catapult and EoS?]
% \label{rmk:dead}
%     For PHB, we sometimes observe a lack of PS after the large catapults, as shown in Figure~\ref{fig:cifar10-exp} (a-c). However, with larger datasets, PS appears to occur again, as shown in Figure~\ref{fig:cifar10-exp}(d-f).
%     In general, the presence or absence of PS after a large catapult seems to depend on the dataset's size and the model architecture.
%     One possible explanation for the lack of PS in ReLU networks is that the large catapults in PHB aggressively reduce the number of active ReLU neurons~\citep{lu2020dying,andriushchenko2023sam}; see Figure~\ref{fig:deadrelu} of Appendix~\ref{app:dead}
% \end{remark}

% \begin{remark}[Flatness and Generalization - future work]
%     We are not claiming that flatter minima~\cite{hochreiter1997flat} necessarily lead to better generalization, and in fact, some ``counterexamples'' have been recently proposed~\citep{dinh2017sharp,even2023diagonal,andriushchenko2023sharpness}. Instead, we solely focus on the effect of momentum and learning rate warmup on how sharpness changes throughout training.
% \end{remark}

\section*{Acknowledgements}
We thank the anonymous reviewers for their helpful comments in improving the paper.
We also thank Alex Damian, Eshaan Nichani, Jeremy Cohen, Jason Lee, and Molei Tao for helpful discussions at the NeurIPS 2023 M3L Workshop.
This project was funded by the 2023 Microsoft Research Asia Collaborative Research grant funded by Microsoft.
J. Lee and C. Yun were supported by the Institute of Information \& Communications Technology Planning \& Evaluation (IITP) grant funded by the Korean government(MSIT) (No.2019-0-00075, Artificial Intelligence Graduate School Program(KAIST)).

% \section*{Impact Statement}
% This paper advances our understanding of heavy-ball momentum, widely used in machine learning.
% As our work is largely theoretical, there are no significant societal consequences of our work.

\bibliographystyle{plainnat}
\bibliography{references}

\newpage
\appendix
\newpage
\tableofcontents
\newpage
\section{Related Works}
\label{app:related-works}

%\subsection{Catapults in (S)GD}
\paragraph{Catapults in (S)GD.}
\citet{lewkowycz2020catapult} are the first to describe the catapult mechanism of GD, the phenomenon of momentary spikes in training loss resulting in lower sharpness. They then analytically prove that catapults occur for $f = d^{-1/2} \vu^T \vv$, where $\vu, \vv \in \sR^d$, when $d$ is sufficiently large.
Since then, the theoretical analysis of the catapult mechanism has been extended to other models such as (neural) quadratic models~\citep{zhu2023catapultquadratic,meltzer2023catapult,agarwala2023eos} and matrix factorization~\citep{wang2022implicit,wang2023catapult} (from the perspective of ``balancing'' effect).
On the empirical side, \citet{zhu2023catapult} study the loss spikes that occur during the catapults of SGD by decomposing the loss into components corresponding to different eigenspaces of the neural tangent kernel (NTK; \citet{jacot2018ntk}) and show that catapults improve generalization through feature learning.
% They observe that the loss spikes occur in the top eigenspace of the NTK while the loss components corresponding to lower eigenspaces are unaffected, resulting in the loss quickly decreasing after the catapults.
% They also show that  and quantitatively measure this using average gradient outer product (AGOP) alignment, and observe that learning rate warmup can induce multiple catapults.
Recent work by \citet{kalra2023catapult} explores the training dynamics of SGD for neural networks for various settings and uncovers four distinct regimes controlled by critical values of the sharpness.
Despite an abundance of works on the catapult mechanism, to the best of our knowledge, we are the first to report the substantial differences in the catapult mechanism and sharpness reduction between GD and PHB.% (in the presence of learning rate warmup).

% \citte{lewkowycz2020catapult} show that GD with large learning rates can induce catapults, resulting in.
% % Several other works \citep{zhu2023catapult, meltzer2023catapult} have also explored the catapult mechanism of GD. 
% The theoretical analysis in \citet{lewkowycz2020catapult} focuses on the network function 
% \citet{zhu2023catapultquadratic,meltzer2023catapult} extend this work by theoretically deriving sufficient conditions for catapults in quadratic nets.

% However, none of the prior works report substantial differences between GD and PHB. 
% They also briefly mention the effects of using momentum, but unlike in our work, they do not 
% One possible explanation is the usage of learning rate warmup in our work which is not present in theirs.
% Furthermore, since our work considers linear learning rate warmup, we also observe the interactions between the sharpness and the MSS during training and note that GD and PHB display qualitatively different behaviors.

%\subsection{Edge of Stability}
\paragraph{Edge of Stability.}
Many works have been devoted to the theoretical analysis of PS and EoS~\citep{ahn2022stability,zhu2023minimalist,ahn2022threshold,wang2022stability,arora2022stability,damian2023stabilization,song2023bifurcation,wu2023logistic,wu2024logistic} as well as their systematic empirical analyses~\citep{gilmer2022stability,cohen2021stability,cohen2022stability}. \citet{ahn2022threshold} rigorously characterize the EoS phenomenon in a single-neuron network. They show that the single-neuron training dynamics under the large learning rate regime display oscillatory ``bouncing'' behavior accompanied by a decrease in sharpness so that the sharpness converges below the MSS. \citet{song2023bifurcation} extend the results of \citet{ahn2022threshold} to two-layer fully connected linear networks trained using a single data point. \citet{damian2023stabilization} explain the EoS phenomenon by proving a ``self-stabilization'' property in GD. By utilizing cubic Taylor expansions of the loss function, they prove that if the sharpness is above the MSS, GD will begin to oscillate and diverge in the leading eigenvector direction. However, the oscillations induce a ``self-stabilization'' effect which moves the iterates in the $-\nabla S$ direction thus reducing the sharpness and stabilizing the dynamics. This interplay between the oscillations and self-stabilization result in the EoS behavior.
There are also works explaining the effect of optimization tricks by considering their interaction with EoS \citep{gilmer2022stability,fu2023momentum}.
Lastly, we note that although the main focus of \citet{cohen2021stability} is on GD, the authors also include PHB experiments, mostly in their Appendix~N. Their experiments also display occasional larger catapults, but this difference is not highlighted.%y do not highlight the difference in the magnitude of the catapults.}
% Although traditional convergence analyses show that training diverges if $\lambda$ (referred to as \textit{sharpness}), the largest eigenvalue of the Hessian of the training loss, exceeds the maximum stable sharpness (MSS) $\text{MSS} = 2/\eta$ where $\eta$ is the learning rate, the recent empirical work by \cite{cohen2021stability} suggests that this is not true in practice. Instead, as the sharpness exceeds the MSS, . This period of oscillatory behavior is referred to as the \textit{Edge of Stability (EoS) Regime} and during the regime, training loss decreases nonmonotonically.

%\subsection{Role of Momentum in Generalization}
\paragraph{Role of Momentum in Generalization.}
Closely related works are \citet{ghosh2023implicit,jelassi2022momentum}, which also consider the positive effects of momentum for nonlinear neural networks.
\citet{jelassi2022momentum} prove that a binary classification setting exists where PHB provably generalizes better than GD for a one-hidden-layer convolutional network.
\citet{ghosh2023implicit} derive an implicit gradient regularizer~\cite{barrett2021implicit} for PHB that biases the solution towards flatter minima, and they showed that the momentum regularizer term is stronger than that of GD by a factor of $\frac{1+\beta}{1 - \beta}$.
In contrast, we report that the behaviors of PHB and GD are fundamentally different.
\newpage
\section{Missing Proofs for Section~\ref{sec:toy}}
\label{app:proofs}

Throughout, recall the following quantities:
\begin{equation}
    \tau_u := \inf\left\{ t \geq 0 : u_t^2 < \frac{2 - \epsilon}{\eta} \right\}, \
    C_u := \frac{u_{\tau_u} - \beta u_{\tau_u - 1}}{1 - \beta}, \
    C_v := \frac{1 + \beta}{1 - \beta} \sum_{t = \tau_u}^\infty v_t^2.
\end{equation}

\subsection{Proof of Lemma~\ref{lem:u}}
\label{app:lemma-u}
Recall the update rule for $u$-coordinate:
\begin{equation}
    u_{t+1} - u_t = \beta (u_t - u_{t-1}) - \eta(1 + \beta) u_t v_t^2 \indicator[u_t \geq 0].
\end{equation}

We proceed by induction.
For $t = 1$, it is trivial.
For $1 < t < \tau_0$, we have that
\begin{equation*}
    u_{t+1} - u_t = \beta (u_t - u_{t-1}) - \eta(1 + \beta) u_t v_t^2
    \leq 0,
\end{equation*}
as $u_t \leq u_{t-1}$ and $\eta u_t v_t^2 > 0$.
If $\tau_0 = \infty$, then by monotone convergence theorem, $u_t$ converges to some $u_\infty \geq 0$.
If not, then for $t \geq \tau_0$, we can solve the recursion to obtain
\begin{equation*}
    u_t = u_{\tau_0} + \frac{\beta (u_{\tau_0 - 1} - u_{\tau_0})}{1 - \beta} (\beta^{t - \tau_0} - 1).
\end{equation*}
As $u_{\tau_0} < u_{\tau_0-1}$ by the definition of $\tau_0$, $u_t$ also monotonically decreases for $t \geq \tau_0$ (in a geometric speed), and we conclude by again applying the monotone convergence theorem.
\qed

\subsection{Proof of Theorem~\ref{thm:toy}}
\label{app:pf-toy}

By telescoping, we have that for any $t \geq \tau_u + 1$,
\begin{equation*}
    u_t - u_{\tau_u} = \beta (u_{t-1} - u_{\tau_u - 1}) - \eta(1 + \beta) \sum_{k = \tau_u}^{t-1} u_k v_k^2.
\end{equation*}
First, let $N \geq \tau_u$ be fixed and define $C_v(N) := \frac{1 + \beta}{1 - \beta} \sum_{t = \tau_u}^N v_t^2.$
Then, for $t \geq N + 1$,
\begin{equation*}
    u_t - \beta u_{t-1} = u_{\tau_u}- \beta u_{\tau_u - 1} - \eta(1 + \beta) \sum_{k = \tau_u}^{t-1} u_k v_k^2
    \leq (1 - \beta) C_u - \eta (1 - \beta) C_v(N) u_t.
\end{equation*}
We can rewrite the recursive inequality as
\begin{equation*}
    u_t - \frac{C_u}{1 + \eta C_v(N)} \leq \frac{\beta}{1 + \eta (1-\beta) C_v(N)} \left( u_{t-1} - \frac{C_u}{1 + \eta C_v(N)} \right).
\end{equation*}
To deal with possibly changing sign, we consider the first time in which the iterates pass another point:
\begin{equation}
    \tau_u'(N) := \inf\left\{ t \geq N + 1 : u_t < \frac{C_u}{1 + \eta C_v(N)} \right\}.
\end{equation}

If $\tau_u'(N) = \infty$, then we have that for all $t \geq N + 1$,
\begin{equation}
    u_t \leq \frac{C_u}{1 + \eta C_v(N)} + \left( \frac{\beta}{1 + \eta (1-\beta) C_v(N)} \right)^{t - N - 1} \left( u_N - \frac{C_u}{1 + \eta C_v(N)} \right).
\end{equation}

If not, then we have that for all $t \geq \tau_u'(N)$,
\begin{equation}
    u_t \leq \frac{C_u}{1 + \eta C_v(N)} - \left( \frac{\beta}{1 + \eta (1-\beta) C_v(N)} \right)^{t - \tau_u'(N)} \left( \frac{C_u}{1 + \eta C_v(N)} - u_{\tau_u'(N)} \right),
\end{equation}

In either case, we obtain the desired conclusion by taking the limit $\min\{N, t\} \rightarrow \infty$ with $t \geq N + 1$.
\qed

\subsection{Proof of Theorem~\ref{thm:toy-GD}}
\label{app:pf-toyGD}
We start by providing a nonasymptotic version of Theorem~\ref{thm:toy-GD}:
\begin{theorem}
    \label{thm:toy-GD-full}
    For sufficiently small $0 < v_0^2 < \epsilon \ll 1$, we have that
    \begin{equation}
        \lim_{t \rightarrow \infty} u_t \geq \sqrt{\frac{2}{\eta}} - \sqrt{P_{\tau_u} \exp\left( \frac{4\eta^2 u_{\tau_u}^2 P_{\tau_u}}{\epsilon(2 - \epsilon)} \right)},
    \end{equation}
    where $P_{\tau_u}$ is a quantity satisfying
    \begin{equation}
        P_{\tau_u} \leq \left( \frac{\epsilon}{\eta} + \frac{1}{2} v_0^2 \right) \exp\left( 2(2 + \epsilon) \sqrt{\frac{2 + \epsilon}{2 - \epsilon}} + 2(2 + \epsilon) \sqrt{\frac{2\epsilon}{2 - \epsilon}} \right).
    \end{equation}
\end{theorem}

% \begin{theorem}
% \label{thm:toy-GD-full}
%     For any $\delta \in (0, 1)$,
%     \begin{equation}
%         \limsup_{t \rightarrow \infty} \bigbrace{\sqrt{\frac{2}{\eta}} - u_t} \leq - \sqrt{\frac{\delta}{\eta}}
%         + 2\sqrt{\frac{(2 - \epsilon)}{\eta}} W(\epsilon) \left( 1 + \frac{\exp\left( 4(2+\epsilon) W(\epsilon) \right)}{\delta ( 2 - \delta)} \right),
%     \end{equation}
%     where $W(\epsilon) := \left( \epsilon + \frac{\eta v_0^2}{2} \right) \exp\left( 2(2 + \epsilon) \sqrt{\frac{2\epsilon}{2 - \epsilon}} \right)$,
% \end{theorem}
% Then, it is easy to see how it implies the asymptotic version in the main text (Theorem~\ref{thm:toy-GD}): for sufficiently small $\epsilon$ and $v_0^2 = \gO(\epsilon)$, and moderately large $\eta$ (such that it can be regarded as constant), we can choose $\delta = \epsilon^{2/3}$.

\begin{proof}[Proof of Theorem~\ref{thm:toy-GD-full}]
    In contrast to the proof technique used for Theorem~\ref{thm:toy}, we utilize an energy argument that works for GD. 
    Inspired by the empirical observation that the GD iterates roughly form an ellipse centered around the point $\left( \sqrt{\frac{2}{\eta}}, 0 \right)$, we consider the following ``elliptical energy'' function:
    \begin{equation}
        P_t := \left( u_t - \sqrt{\frac{2}{\eta}} \right)^2 + \frac{1}{2} v_t^2.
    \end{equation}
    Note that $P_0 \leq \frac{\epsilon}{\eta} + \frac{1}{2} v_0^2$.
    We will first prove that the elliptical energy is well-bounded and then use that fact to lower bound $u_\infty$.

    Let us first fix $\epsilon, v_0^2 \in (0, 1)$.
    The following key lemma, whose proof is provided at the end of this section, states that the energy is approximately well-bounded, given that $u_t$ is sufficiently lower bounded:
    \begin{lemma}
    \label{lem:Pt-exp-bound}
        $P_{t+1} \leq P_t \exp\left( 2\eta^2 u_t^2 v_t^2 \right)$ for any $t \geq 0$ satisfying $u_t \ge \frac{1}{\sqrt{\eta}}$.
    \end{lemma}
    As $u_t^2 \geq u_{\tau_u-1}^2 \ge \frac{2-\epsilon}{\eta} > \frac{1}{\eta}$ for any $t \leq \tau_u - 1$ (due to Lemma~\ref{lem:u}), we have:
    \begin{align*}
        P_{\tau_u} &\leq \left( \frac{\epsilon}{\eta} + \frac{1}{2} v_0^2 \right) \exp\left( 2\eta^2 \sum_{t = 0}^{\tau_u - 1} u_t^2 v_t^2 \right) \tag{telescoping with Lemma~\ref{lem:Pt-exp-bound}} \\
        &\leq \left( \frac{\epsilon}{\eta} + \frac{1}{2} v_0^2 \right) \exp\left( 2\eta^2 u_0^2 \sum_{t = 0}^{\tau_u - 1} v_t^2 \right) \tag{Lemma~\ref{lem:u}} \\
        &= \left( \frac{\epsilon}{\eta} + \frac{1}{2} v_0^2 \right) \exp\left( 2(2 + \epsilon)\eta v_{\tau_u - 1}^2 + 2(2 + \epsilon) \sum_{t = 0}^{\tau_u - 2} \frac{u_t - u_{t+1}}{u_t} \right) \tag{$u_0 = \sqrt{\frac{2 + \epsilon}{\eta}}$, $u_t - u_{t+1} = \eta v_t^2 u_t$} \\
        &\leq \left( \frac{\epsilon}{\eta} + \frac{1}{2} v_0^2 \right) \exp\left( 2(2 + \epsilon)\eta v_{\tau_u - 1}^2 +  \frac{2 \sqrt{\eta} (2 + \epsilon)}{\sqrt{2 - \epsilon}} \sum_{t = 0}^{\tau_u - 2} (u_t - u_{t+1}) \right) \tag{$u_t^2 \geq \frac{2 - \epsilon}{\eta}$ for $t \leq \tau_u - 1$} \\
        &= \left( \frac{\epsilon}{\eta} + \frac{1}{2} v_0^2 \right) \exp\left( 2(2 + \epsilon)\eta v_{\tau_u - 1}^2 + \frac{2 \sqrt{\eta} (2 + \epsilon)}{\sqrt{2 - \epsilon}} (u_0 - u_{\tau_u - 1}) \right) \\
        &\leq \left( \frac{\epsilon}{\eta} + \frac{1}{2} v_0^2 \right) \exp\left( 2(2 + \epsilon)\eta v_{\tau_u - 1}^2 + 2(2 + \epsilon) \sqrt{\frac{2\epsilon}{2 - \epsilon}} \right) \tag{$\sqrt{a + b} - \sqrt{b} \leq \sqrt{a}$}.
    \end{align*}
    Also, we have that
    \begin{equation}
    \label{eqn:v-tau}
        v_{\tau_u - 1}^2 = \frac{u_{\tau_u - 1} - u_{\tau_u}}{\eta u_{\tau_u - 1}}
        \leq \frac{u_0}{\eta \sqrt{\frac{2 - \epsilon}{\eta}}}
        = \sqrt{\frac{2 + \epsilon}{2 - \epsilon}} \frac{1}{\eta}.
    \end{equation}
    Thus, we have that
    \begin{equation}
    \label{eqn:P-tau}
        P_{\tau_u} \leq \left( \frac{\epsilon}{\eta} + \frac{1}{2} v_0^2 \right) \exp\left( 2(2 + \epsilon) \sqrt{\frac{2 + \epsilon}{2 - \epsilon}} + 2(2 + \epsilon) \sqrt{\frac{2\epsilon}{2 - \epsilon}} \right).
    \end{equation}
    % \begin{corollary}
    % \label{cor:v-tau}
    %      For $\epsilon < 1$, $v_{\tau_u - 1}^2 \leq \left( \frac{\epsilon}{\eta} + \frac{1}{2} v_0^2 \right) \exp\left( 2(2 + \epsilon) \sqrt{\frac{2\epsilon}{2 - \epsilon}} \right)$.
    % \end{corollary}
    We now claim that
    \begin{equation}
        u_t \geq \frac{1}{\sqrt \eta} \ \text{ and } \ P_t \leq P_{\tau_u} \exp\left( \frac{4\eta^2 u_{\tau_u}^2 P_{\tau_u}}{\epsilon(2 - \epsilon)} \right), \quad \forall t \geq \tau_u - 1,
    \end{equation}
    which then implies our desired statement.

    We proceed by induction.
    The base case ($t = \tau_u - 1$) is trivial.
    For $t' \geq \tau_u$, suppose the statement holds for all $t < t'$.
    Again, using Lemma~\ref{lem:Pt-exp-bound}, we have that
    \begin{equation*}
        P_{t'} \leq P_{\tau_u} \exp\left( 2\eta^2 \sum_{t=\tau_u}^{t' - 1} u_t^2 v_t^2 \right)
        \leq P_{\tau_u} \exp\left( 2\eta^2 u_{\tau_u}^2 \sum_{t=\tau_u}^{t' - 1} v_t^2 \right).
    \end{equation*}
    
    Let us now bound $\sum_{t = \tau_u}^{t' - 1} v_t^2$.
    For $t \in [\tau_u, t' - 1]$, we have that $(\eta u_t^2 - 1)^2 < (1 - \epsilon)^2$ by the induction hypothesis, and thus, $v_{t+1}^2 < (1 - \epsilon)^2 v_t^2$.
    This implies that
     \begin{equation*}
        \sum_{t = \tau_u}^{t' - 1} v_t^2 < v_{\tau_u}^2 \sum_{t = 0}^{t' - \tau_u - 1} (1 - \epsilon)^{2t}
        \leq v_{\tau_u}^2 \sum_{t = 0}^\infty (1 - \epsilon)^{2t}
        = \frac{v_{\tau_u}^2}{\epsilon(2 - \epsilon)}
        \leq \frac{2 P_{\tau_u}}{\epsilon(2 - \epsilon)},
    \end{equation*}
    and thus, we have that $P_{t'} \leq P_{\tau_u} \exp\left( \frac{4\eta^2 u_{\tau_u}^2 P_{\tau_u}}{\epsilon(2 - \epsilon)} \right)$.
    From the definition of our elliptical energy, this then implies that
    \begin{equation}
        u_{t'} \geq \sqrt{\frac{2}{\eta}} - \sqrt{P_{\tau_u} \exp\left( \frac{4\eta^2 u_{\tau_u}^2 P_{\tau_u}}{\epsilon(2 - \epsilon)} \right)}.
    \end{equation}
    As $P_{\tau_u} = \gO(\epsilon)$ for small $\epsilon$ and $v_0^2 = \gO(\epsilon)$, with suitable choices we can conclude that $u_{t'} \geq \sqrt{\frac{1}{\eta}}$, and we are done.
    % \begin{align*}
    %     v_{\tau_u}^2 &\le 2  \\
    %     &\leq \exp\left( 2(2+\epsilon) \eta v_{\tau_u-1}^2 \right) P_{\tau_u - 1} \tag{Lemma~\ref{lem:Pt-exp-bound}} \\
    %     &\le 2 P_0 \exp\left( 4(2+\epsilon) \eta v_{\tau_u-1}^2 + 2(2+\epsilon)\sqrt{\frac{2\epsilon}{2 - \epsilon}} \right) \tag{Lemma~\ref{lem:Pt-exp-bound}} \\
    %     &\le 2 P_0 \exp\left( 4(2+\epsilon) \eta P_0 \exp\left( 2(2 + \epsilon) \sqrt{\frac{2\epsilon}{2 - \epsilon}} \right) + 2(2+\epsilon)\sqrt{\frac{2\epsilon}{2 - \epsilon}} \right) \tag{Corollary~\ref{cor:v-tau}} \\
    %     &= 2 \left( \frac{\epsilon}{\eta} + \frac{v_0^2}{2} \right) \underbrace{\exp\left( 4(2+\epsilon) \left( \epsilon + \frac{\eta v_0^2}{2} \right) \exp\left( 2(2 + \epsilon) \sqrt{\frac{2\epsilon}{2 - \epsilon}} \right) + 2(2+\epsilon)\sqrt{\frac{2\epsilon}{2 - \epsilon}} \right)}_{:=C(\epsilon, v_0)}.
    % \end{align*}
    % Thus $P_t \le \left( \frac{\epsilon}{\eta} + \frac{v_0^2}{2} \right) C(\epsilon, v_0) < \left( \frac{\epsilon}{\eta} + \frac{v_0^2}{2} \right) C(1,1)$ for all $t$ such that $u_{t-1} \ge 1/\sqrt{\eta}$. \\
\end{proof}

\begin{proof}[Proof of Lemma~\ref{lem:Pt-exp-bound}]
    This is shown via a brute-force computation:
    \begin{align*}
        P_{t+1} &= \left( u_{t+1} - \sqrt{\frac{2}{\eta}} \right)^2 + \frac{1}{2} v_{t+1}^2 \\
        &= \left( u_t - \eta u_t v_t^2 - \sqrt{\frac{2}{\eta}} \right)^2 + \frac{1}{2} \left( v_t - \eta v_t u_t^2 \right)^2 \tag{GD update} \\
        &= P_t - 2\eta u_t v_t^2 \left(u_t - \sqrt{\frac{2}{\eta}}\right) + \eta^2 u_t^2 v_t^4 - \eta v_t^2 u_t^2 + \frac{\eta^2}{2} v_t^2 u_t^4 \\
        &= P_t + u_t^2 v_t^2 \left( \eta^2 v_t^2 + \frac{\eta^2}{2} u_t^2 - 3\eta + \frac{2\sqrt{2\eta}}{u_t} \right).
        % &= (1 + 2\eta^2 u_t^2 v_t^2) P_t + 2\eta^2 u_t^2 v_t^2 \left( \frac{1}{4} u_t^2 - \frac{3}{2\eta} + \frac{\sqrt{2}}{\eta\sqrt{\eta} u_t} - \left( u_t - \sqrt{\frac{2}{\eta}} \right)^2 \right).
        % &\leq (1 + 2\eta^2 u_t^2 v_t^2) P_t + 2\eta^2 u_t^2 v_t^2 \left( -\frac{3}{4} u_t^2 - \frac{7 - 2\sqrt{2}}{2\eta} + \frac{2\sqrt{2} u_t}{\sqrt{\eta}} \right) \tag{$u_t \geq \sqrt{\frac{1}{\eta}}$} \\
        % &\leq (1 + 2\eta^2 u_t^2 v_t^2) P_t + 2\eta^2 u_t^2 v_t^2 \left\{ -\frac{3}{4} \left( u_t - \frac{4}{3} \sqrt{\frac{2}{\eta}} \right)^2 + \frac{6\sqrt{2} - 5}{6\eta} \right\} \\
        % &\leq (1 + 2\eta^2 u_t^2 v_t^2) P_t + 2\eta^2 u_t^2 v_t^2 \left( -\frac{3}{4} u_0^2 - \frac{7 - 2\sqrt{2}}{2\eta} + \frac{2\sqrt{2} u_0}{\sqrt{\eta}} \right) \tag{Lemma~\ref{lem:u}} \\
        % &\leq (1 + 2\eta^2 u_t^2 v_t^2) P_t  \tag{$u_0^2 = \frac{2+\epsilon}{\eta}$} \\
        % &\leq P_t \exp\left( 2\eta^2 u_t^2 v_t^2 \right). \tag{$1 + z \leq e^z, \ \forall z \in \sR$}
    \end{align*}
    We then have the following helpful inequality, whose proof is deferred to the end:
    \begin{lemma}
    \label{lem:inequality}
        For $z \geq \frac{1}{\sqrt{\eta}}$, $\frac{\eta^2}{2} z^2 - 3\eta + \frac{2\sqrt{2\eta}}{z} \leq 2\eta^2 \left( z - \sqrt{\frac{2}{\eta}} \right)^2$.
    \end{lemma}
    Using this and the given assumption that $u_t \geq \frac{1}{\sqrt{\eta}}$, we then have the desired statement as follows:
    \begin{equation*}
        P_{t+1} \leq P_t + u_t^2 v_t^2 \left( \eta^2 v_t^2 + 2\eta^2 \left( u_t - \sqrt{\frac{2}{\eta}} \right)^2 \right)
        = (1 + 2 \eta^2 u_t^2 v_t^2) P_t.
    \end{equation*}
\end{proof}
\begin{proof}[Proof of Lemma~\ref{lem:inequality}]
    By reparametrizing $z \gets z / \sqrt{\eta}$, it suffices to prove that for $z \geq 1$, $\frac{1}{2} z^2 - 3 + \frac{2\sqrt{2}}{z} \leq 2 \left( z - \sqrt{2} \right)^2.$
    By rearranging, this is equivalent to
    \begin{equation*}
        f(z) \triangleq 3z^3 - 8\sqrt{2} z^2 + 14z - 4\sqrt{2} \geq 0, \quad \forall z \geq 1.
    \end{equation*}
    This is then obvious, as $f$ is a cubic function with $f(0) > 0$, the local minimum of $(\sqrt{2}, 0)$ and the local maximum of $\left( \frac{7\sqrt{2}}{9}, \frac{8\sqrt{2}}{243} \right)$.
\end{proof}
\newpage

\newpage
\section{Additional Experimental Results}
\label{app:additional}
\paragraph{Additional Experimental Details} Experiments on nonlinear networks were carried out using an A6000. All other experiments were run locally on CPU. Experiments on nonlinear networks are based on code from \url{https://github.com/locuslab/edge-of-stability}~\citep{cohen2021stability}.

\subsection{Is Linear Warmup Necessary?}
\label{app:warmup-necessity}
\begin{figure}[!t]
    \centering
    \subfigure[Step warmup.]{
        \includegraphics[width=0.45\linewidth]{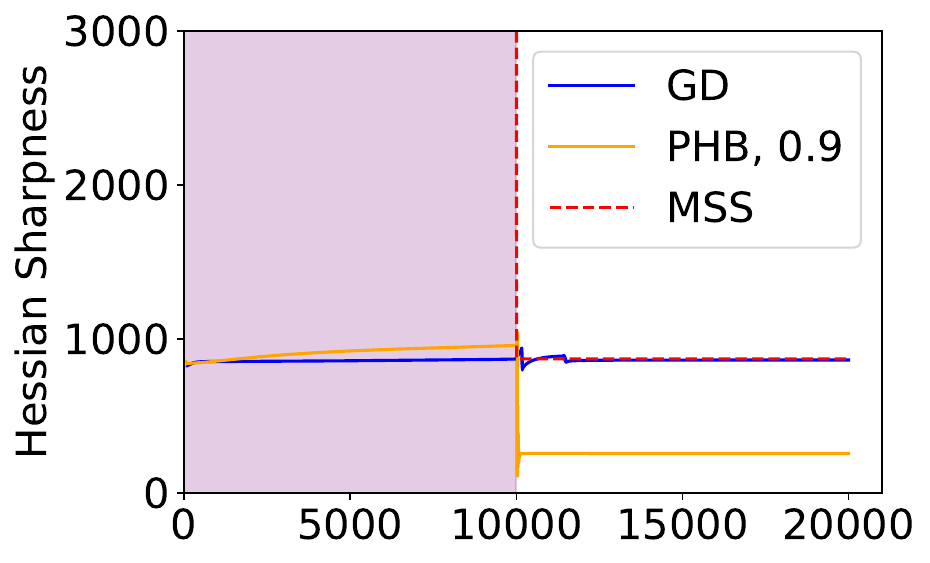}
        \label{fig:ldn-step-warmup}
    }
    \subfigure[Terminating the warmup.]{
        \includegraphics[width=0.45\linewidth]{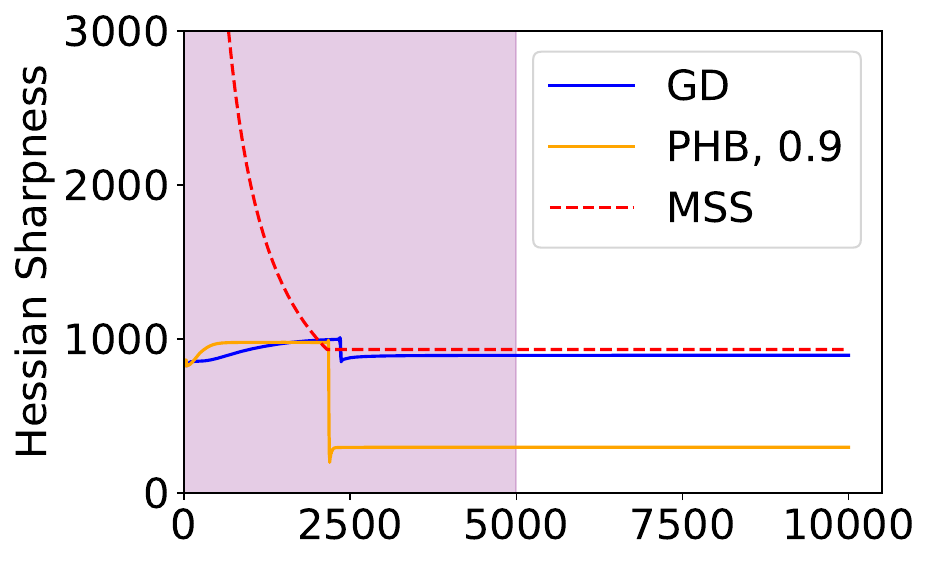}
        \label{fig:warmup-stop}
    }
    \caption{Ablations on the learning rate warmup. (a) Step warmup is used instead of linear warmup. (b) The warmup period is initially set to $5000$ iterations, but the warmup is terminated at iteration $2150$ around the iteration where the sharpness crosses the MSS.}
    \label{fig:ldn-warmup}
\end{figure}

Although we use linear warmup in our experiments, we emphasize that linear warmup is {\it not} necessary to induce the catapults.
As mentioned in the main text, we observe that the main criteria for inducing the catapults are (1) for the iterates to be in a neighborhood of a stable minimum and (2) for the current learning rate to be large enough that the minimum is unstable (in that the sharpness of the minimum is above the MSS) but not so unstable that training diverges.
% and are (1) convergence to an initial (sharp) minimum and (2)
Linear warmup satisfies these two criteria by (1) stably moving the iterates towards a minimum under a low learning rate and (2) automatically finding a suitably large learning rate (that does not lead to divergence) by gradually increasing the learning rate. However, as long as the two criteria are met, catapults can be induced without linear warmup.

\paragraph{Other Warmup in the LDN.}
To show that the specific form of the warmup is not essential in inducing the catapults, we train an LDN using a step warmup scheduler.
We use the learning rate of $10^{-5}$ for the first $10000$ iterations and then $0.0023$ for the remaining $10000$ iterations. Here, it is necessary to use a sufficiently long warmup period to ensure that the pre-catapult training loss is close to zero. 
% Note that under the same learning rate, the MSS of GD and PHB differ by a factor of $(1+\beta)$.
%Since both GD and PHB are initialized at the same $\alpha \cdot \bm{1}$, the learning rate used for the PHB experiment is the specified learning rate multiplied by $(1+\beta)$ to match the MSS and ensures that if the sharpness hits the MSS for GD, it also hits the MSS for PHB.
As shown in Figure~\ref{fig:ldn-step-warmup}, this setting also induces a catapult despite not using linear warmup.
It should be noted that, unlike linear warmups, the final learning rate must be carefully tuned to prevent training from diverging. 
%We can interpret step warmup as a two-phase process of (1) finding a stable minimum by training under low $\eta$ and then (2) inducing a catapult by using a constant but sufficiently large $\eta$.
%{\color{red}Notably, this suggests that if the parameters are initialized close to a minimum and $\eta$ is initialized so that the initial sharpness is only slightly above the MSS, large catapults can be induced while keeping $\eta$ constant as in the second phase of step warmup. Based on this observation, we consider a constant (but large) $\eta$ setting to simplify the analysis in Section \ref{sec:prolong}.}
% (unlike linear warmup, which ``smoothly'' leads to a point that induces catapults): if the learning rate is too small, the catapult is not induced, and if it is too large, training diverges.

To show the effectiveness of the linear warmup in finding the appropriate scale of the learning rate for inducing catapults, we terminate the warmup as soon as the sharpness crosses the MSS, even before the prescribed warmup period ends.
As shown in Figure \ref{fig:warmup-stop}, this is enough to induce catapults for PHB, supporting our claim that linear warmup has the advantage of ``smoothly'' finding a suitable learning rate for inducing catapults.

\paragraph{Linear Warmup in the Toy Model.}
Conversely, although we use a fixed learning rate for the toy model, we show in Figure~\ref{fig:toy-model-warmup} that catapults still occur even when using linear warmup.

\begin{figure}[!h]
    \centering
    \includegraphics[width=0.7\textwidth]{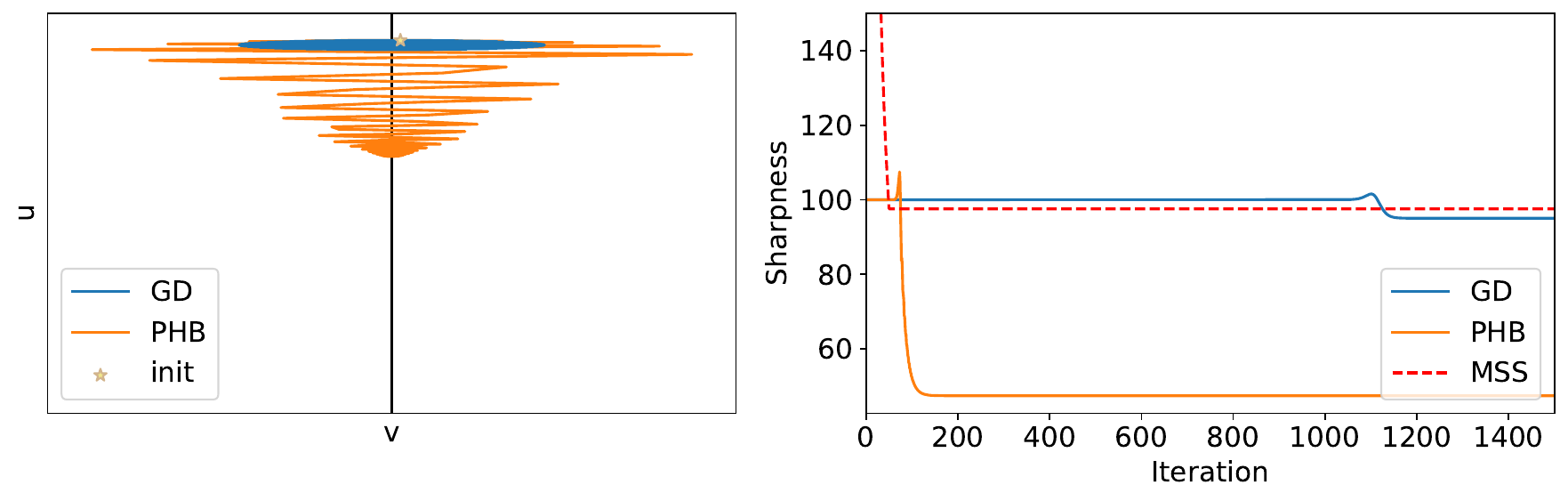}
    \caption{Training the toy model using linear learning rate warmup still induces a catapult}
    \label{fig:toy-model-warmup}
\end{figure}

\subsection{Effects of ``Overshooting'' by Larger Catapults}
\label{app:overshooting}
When the momentum parameter $\beta$ is large, and an even larger catapult can happen and PHB may carry the iterates farther than the flattest minima, resulting in ``overshooting.''
% and a lesser degree of sharpness reduction. 
A good example can be found in the simple LDN loss $\gL(u,v) = \frac{1}{2}(u^2-v^2-1)^2$. As shown in Figure~\ref{subfig:ldn-overshooting}, momentum with $\beta = 0.99$ causes the iterates to move past the flattest minima $(u,v)=(1,0)$ and converge at a sharper solution. Overshooting may also explain why the final sharpness increases again in Figure~\ref{fig:LDN-main-phb} as $\alpha$ increases further: there is some optimal $\alpha$ for which momentum allows the iterate to reach the flattest possible minima, but increasing $\alpha$ even further may result in overshooting.

\begin{figure}[!h]
    \centering
    \subfigure[Theoretical lower bound on sharpness reduction based on toy example theory. At $\beta=0.99$, $\Delta S$ slightly decreases due to overshooting.]{
        \includegraphics[width=0.35\textwidth]{figures/ldn/ldn-2d-ub.pdf}
        \label{subfig:ldn-ub-bound}
    }
    \hfill
    \subfigure[Trajectory plot for $\beta=0.99$. The iterates overshoot past $(u,v)=(1,0)$ which is the flattest minima.]{
        \includegraphics[width=0.25\textwidth]{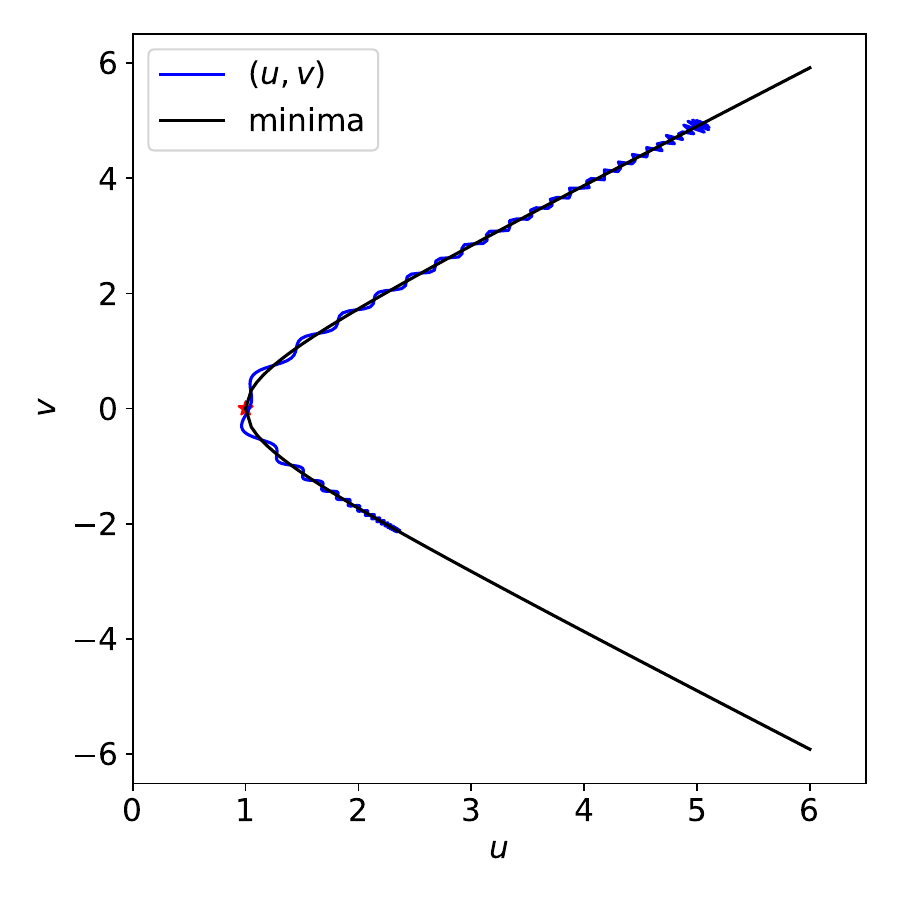}
        \label{subfig:ldn-overshooting}
    }
    \hfill
    \subfigure[In the {\bf ReLU Scalar Network}, overshooting cause by momentum can lead to the neuron dying.]{
        \includegraphics[width=0.25\textwidth]{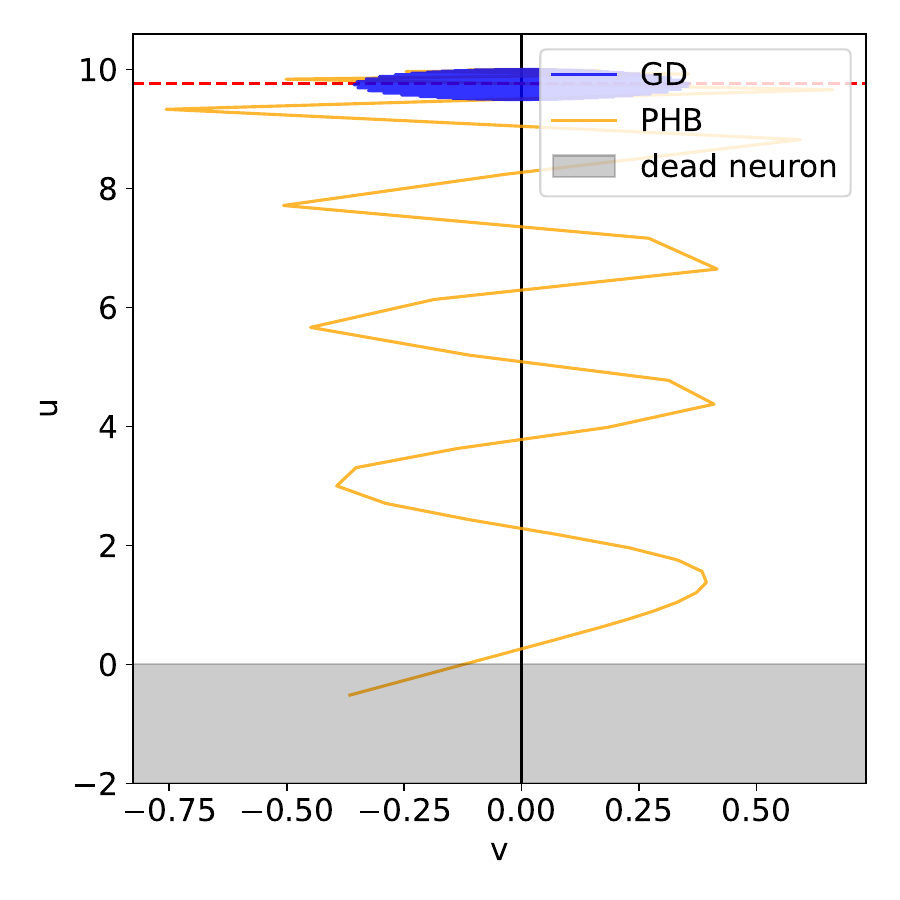}
        \label{subfig:scalar-overshooting}
    }
    \caption{Overshooting in simple models}
    \label{fig:simple-overshooting}
\end{figure}

% \subsection{Large Catapults Lead to More Dead Neurons}
% \label{app:dead}
As the loss landscape and manifold of minima become increasingly complex, overshooting may result in different outcomes, depending on the architectures. 
Another effect of overshooting relates to the lack of progressive sharpening in some of the experiments. 
In some of our experiments (Figures~\ref{fig:cifar10-exp}(a) and \ref{fig:cifar10-exp}(b)), we observed an interesting phenomenon where PHB does not display progressive sharpening after a few large catapults.
We conjecture this lack of progressive sharpening to be a result of large catapult inducing more dead neurons. Indeed, through a synthetic experiment with ReLU FCN, we show that large catapults due to PHB induce more dead neurons than GD; see Figure~\ref{fig:deadrelu}. Although seemingly unrelated, overshooting is one possible explanation as to why momentum induces more dead neurons. To illustrate this point, consider the case of the ReLU scalar network in Figure~\ref{subfig:scalar-overshooting}. Under certain settings, momentum can cause the iterates to overshoot the flattest possible minima $(u=0)$ and land in the region $u < 0$ which kills the ReLU neuron. A similar phenomenon could occur in more realistic networks as well.
\begin{figure}[!ht]
    \centering
    \subfigure[Sharpness.]{
        \includegraphics[width=0.46\linewidth]{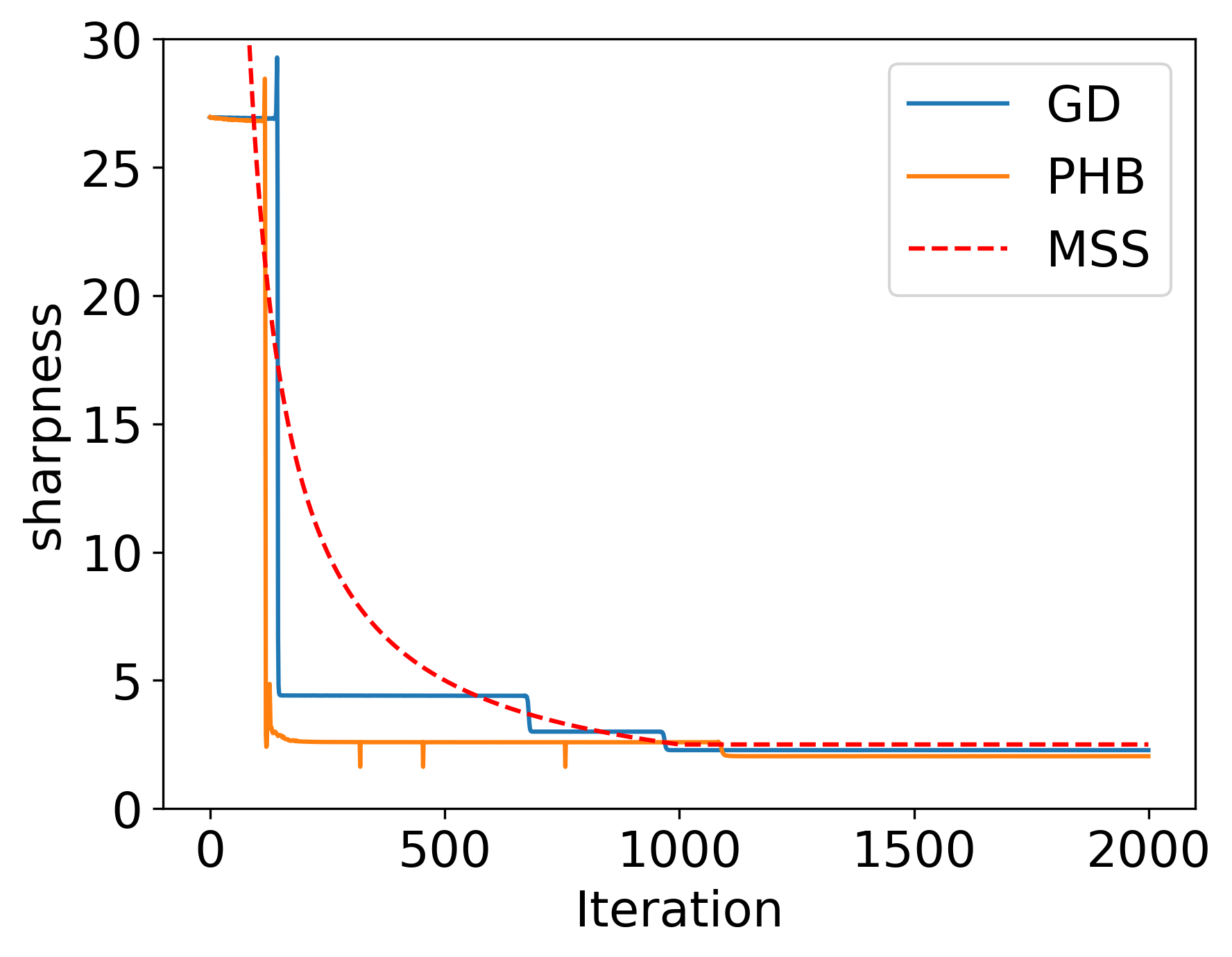}
        \label{fig:deadrelu-sharpness}
    }
    \subfigure[Number of active neurons.]{
        \includegraphics[width=0.46\linewidth]{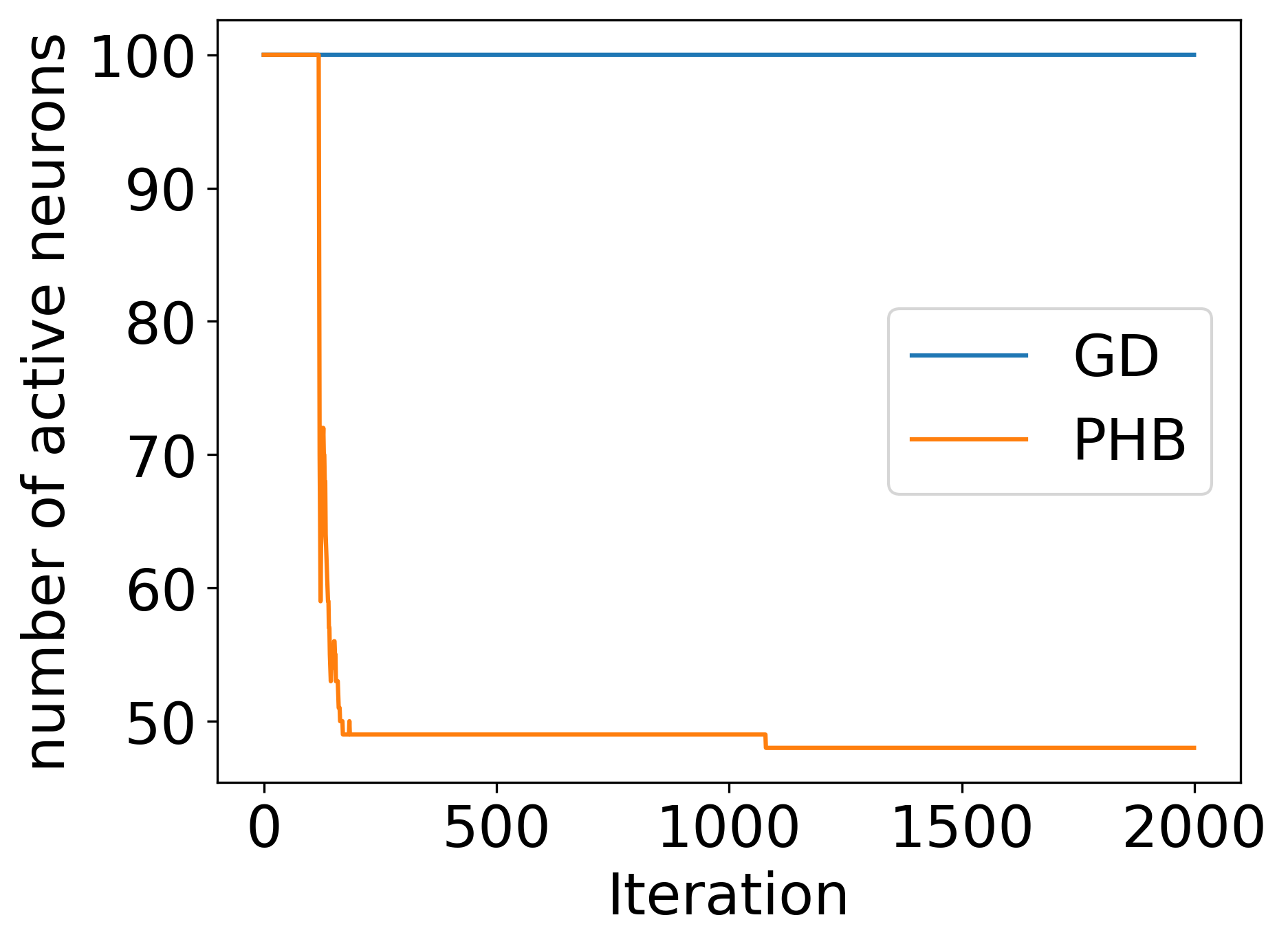}
        \label{fig:deadrelu-neurons}
    }
    \caption{2-layer FCN of width $100$ trained with MSE loss and rank-2 synthetic dataset~\cite{zhu2023catapult}, generated as follows: $\{(\vx_i, y_i)\}_{i=1}^{100}$ with $\vx_i \sim \gN_3(\vzero, \mI)$ and $y_i = (\vx_i)_1 (\vx_i)_2$.}
    \label{fig:deadrelu}
\end{figure}

\subsection{More experiments on Nonlinear Neural Networks}
\label{app:modern}

\subsubsection{Additional FCN Experiments}
For the nonlinear neural network experiment in Figure~\ref{fig:rank2}, we follow the setting of \citet{zhu2023catapult}.
We train a fully-connected 3-layer ReLU network of width $64$ on the synthetic rank-2 dataset.
The synthetic rank-2 dataset is generated by i.i.d. sampling data $\pmb{x}_i \sim \mathcal{N}(\pmb{0}, \mI_{d})$ and generating outputs $y_i = \pmb{x}_i^{(1)} \pmb{x}_i^{(2)}$ (product of the first two coordinates of $\pmb{x}_i$).
A rank2-$D$-$N$ dataset refers to the synthetic rank-2 dataset generated using $d=D$ whose training set consists of $N$ data points; in our experiment, we used a rank2-400-200 dataset.
% For the second nonlinear experiment in Figure \ref{subfig:resnet20_cifar10-1k_10class}, we use a ResNet20 with a 1k-datapoint, 10-class subset of CIFAR10 (Figure~\ref{fig:resnet20_cifar-128-1k} shows the result for a 128-datapoint, 2-class subset of CIFAR10.)
For both experiments, we used a momentum rate of $\beta=0.9$ and MSE loss. \\
% We use a fully-connected 3-layer ReLU network with constant width and train on two datasets using MSE loss: (1)  and (2) a synthetic rank-2 dataset. The simplified CIFAR10 dataset contains 128 training images taken from only the first 2 classes of the full CIFAR10 dataset. The synthetic rank-2 dataset is generated by i.i.d. sampling data $\pmb{x}_i \sim \mathcal{N}(\pmb{0}, \mI_{d})$ and generating outputs $y_i = \pmb{x}_i^{(1)} \pmb{x}_i^{(2)}$ (product of the first two coordinates of $\pmb{x}_i$). In Figure~\ref{fig:rank2}, we used a rank2-400-200 dataset, meaning that we used $d=400$ and a training set with 200 data points. In general, a rank2-$D$-$N$ dataset refers to the synthetic rank-2 dataset generated using $d=D$ whose training set consists of $N$ data points. We used a momentum rate of $\beta=0.9$ for all experiments.

\begin{figure*}[!ht]
\centering
\includegraphics[width=0.31\textwidth]{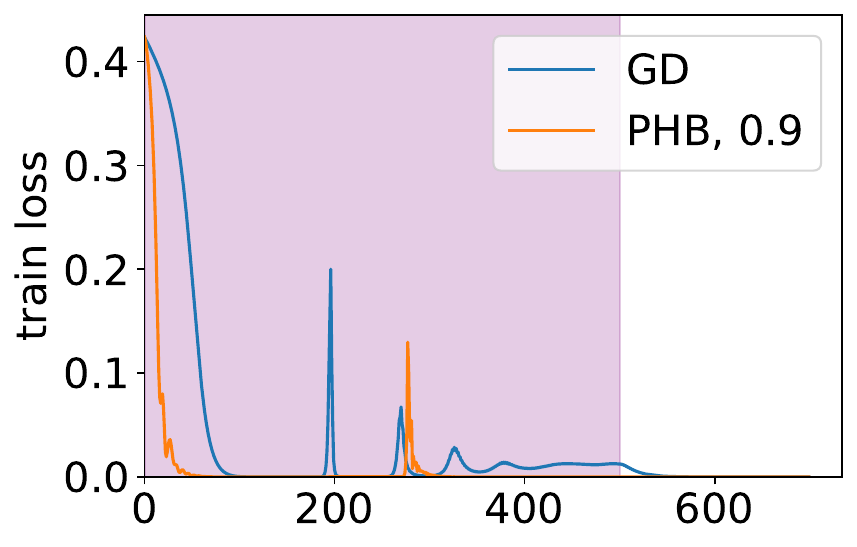}
\hfill
\includegraphics[width=0.31\textwidth]{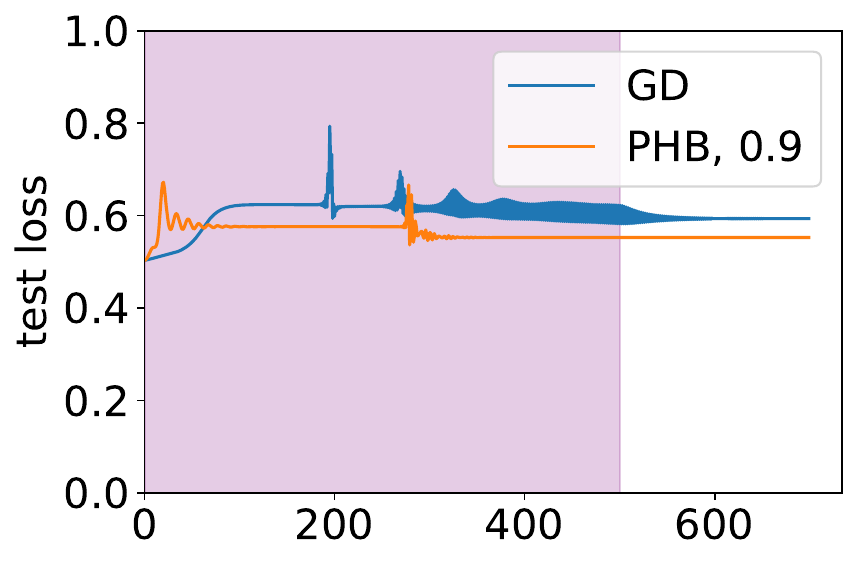}
\hfill
\includegraphics[width=0.3\textwidth]{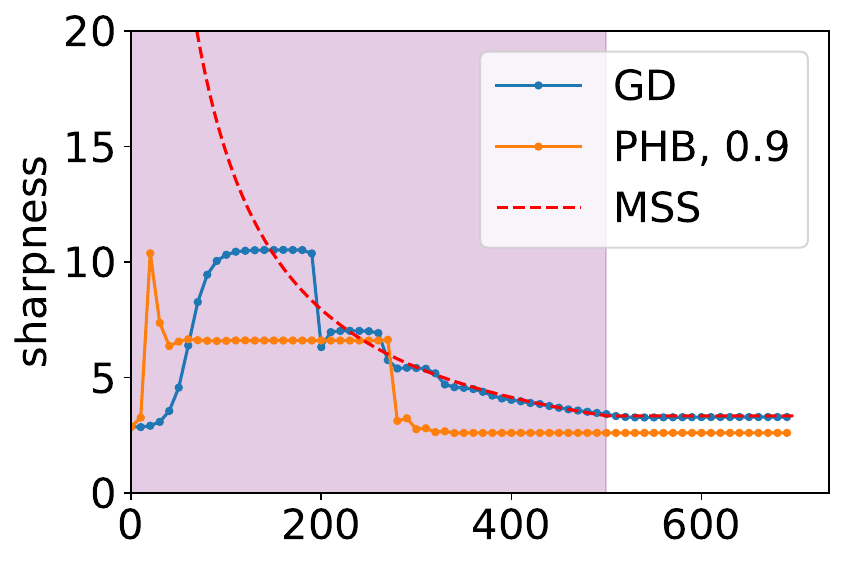}

% \subfigure[CIFAR10-128, width=$256$, $\eta_i = 0.01$, $\eta_f = 0.2$]{
% \includegraphics[width=0.32\textwidth]{figures/sec3_cifar10_train_loss_lin.pdf}
% \hfill
% \includegraphics[width=0.32\textwidth]{figures/sec3_cifar10_test_loss_lin.pdf}
% \hfill
% \includegraphics[width=0.32\textwidth]{figures/sec3_cifar10_sharpness_lin.pdf}
% \label{fig:sharpness}
% }

\caption{FCN trained on the Rank2-400-200 dataset, $\eta_i = 0.02$, $\eta_f = 0.6$.
% Experiments using a 3-layer FCN and a ResNet20, both trained using MSE loss. The purple-shaded region represents the prescribed learning rate warmup period.
}
\label{fig:rank2}
\end{figure*}

% \begin{figure}[!ht]
%     \centering
%     \subfigure[ResNet20 trained on a 2-class 128-datapoint subset of CIFAR10, $\eta_i=0.01$, $\eta_f=0.02$]{
%         \includegraphics[width=0.32\textwidth]{figures/resnet20/cifar10-128_2class/lr0.01-0.02_mse/train_loss_log.pdf}
%         \includegraphics[width=0.32\textwidth]{figures/resnet20/cifar10-128_2class/lr0.01-0.02_mse/test_loss_log.pdf}
%         \includegraphics[width=0.32\textwidth]{figures/resnet20/cifar10-128_2class/lr0.01-0.02_mse/sharpness.pdf}
%     }
%     \subfigure[ResNet20 trained on a 1k-datapoint subset of CIFAR10, $\eta_i = 0.05$, $\eta_f = 0.1$]{
%         \includegraphics[width=0.32\textwidth]{figures/resnet20/cifar10-1k_10class/lr0.05-0.1_mse/train_loss.pdf}
%         \includegraphics[width=0.32\textwidth]{figures/resnet20/cifar10-1k_10class/lr0.05-0.1_mse/test_loss.pdf}
%         \includegraphics[width=0.32\textwidth]{figures/resnet20/cifar10-1k_10class/lr0.05-0.1_mse/sharpness.pdf}
%         \label{subfig:resnet20_cifar10-1k_10class}
%     }
%     \caption{ResNet20 experiments on CIFAR10 with MSE loss}
%     \label{fig:resnet20_cifar_mse}
% \end{figure}

\subsubsection{ResNet20 Experiments}
In Figure~\ref{fig:resnet}, we provide additional experiments on deep neural networks using the ResNet20 architecture. We observe that large catapults also occur for ResNet20 as well. \\

\begin{figure*}[!ht]
    \subfigure[2-class 128 subset of CIFAR10, $\eta_i=0.01$, $\eta_f=0.02$]{
        \includegraphics[width=0.29\textwidth]{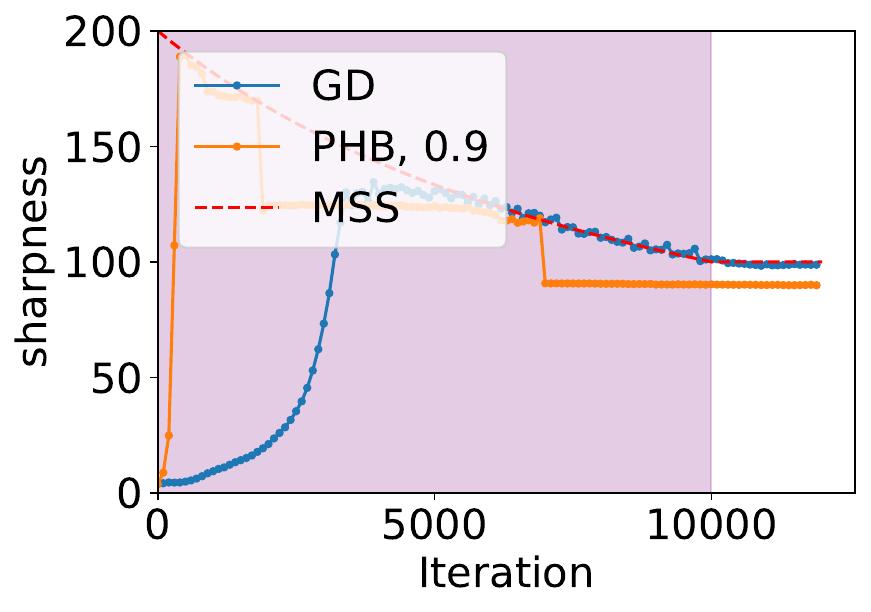}
    }
    \hfill
    \subfigure[1k subset of CIFAR10, $\eta_i=0.01$, $\eta_f=0.1$]{
        \includegraphics[width=0.31\textwidth]{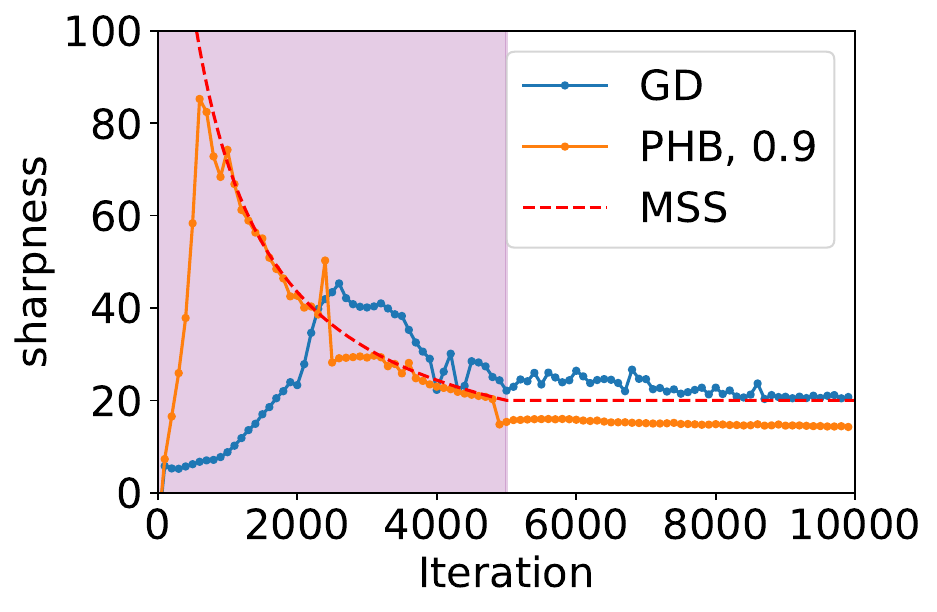}
    } 
    \hfill
    \subfigure[5k subset of CIFAR10, $\eta_i=0.01$, $\eta_f=0.1$]{
        \includegraphics[width=0.31\textwidth]{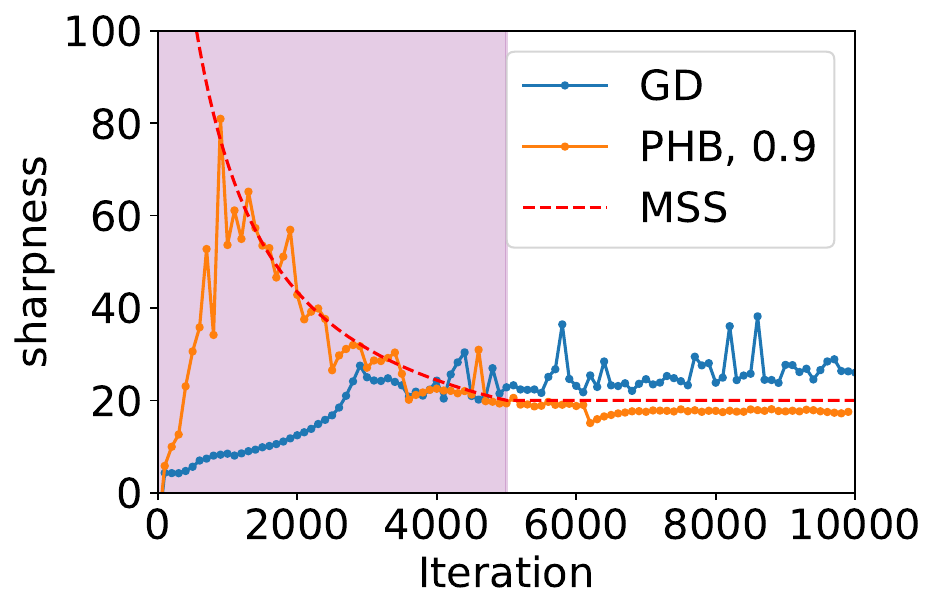}
    }
    \caption{Results for ResNet20. The shaded region is the linear warmup period. 
    }
    \label{fig:resnet}
\end{figure*}

\subsubsection{Cross-Entropy Loss Experiments}
\label{app:ce-exp}
We provide additional experiments showing that large catapults still occur for networks trained using cross-entropy loss. As shown in Figure~\ref{fig:ce-fcn}, for PHB, large catapults occur during early training for FCN trained using Tanh activation and cross-entropy loss. Due to using cross-entropy loss, the iterates also converge quickly to a minimum with near-zero sharpness after the catapult. \\ 

Furthermore, for PHB, the iterates quickly converge to a minimum with near-zero sharpness right after the large catapult whereas for GD the iterates do not immediately converge to a flat minimum after the catapult. Instead, the sharpness of GD iterates settles at the MSS after the catapult before slowly converging towards flatter minima as training progresses.
\begin{figure}[!h]
    \centering
    \includegraphics[width=0.5\textwidth]{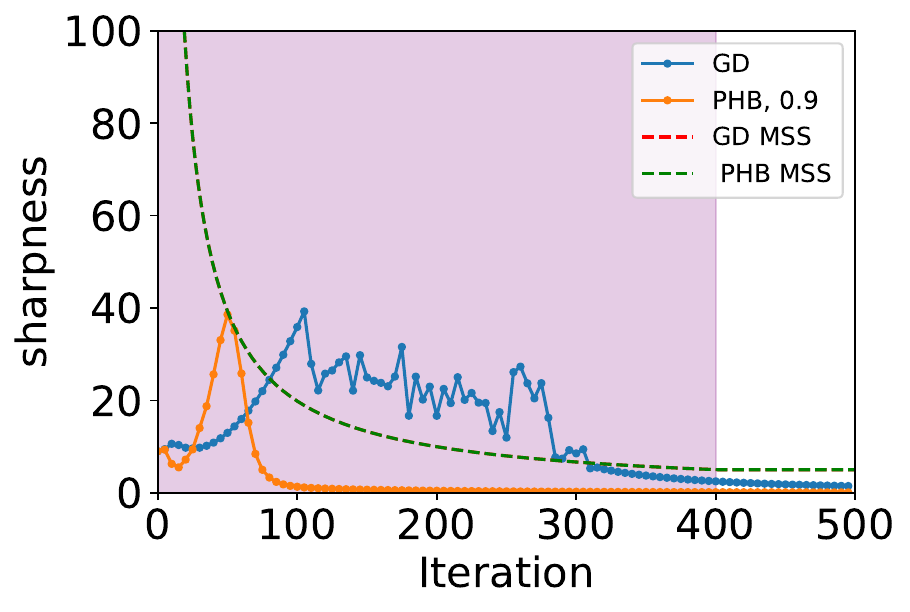}
    \caption{FCN Experiments using \emph{cross-entropy loss} on CIFAR10-1k, with Tanh activation, width 100, $\eta_i=0.001$, $\eta_f=0.4$}
    \label{fig:ce-fcn}
\end{figure}

\newpage
\subsection{Additional Results for $\beta = 0.99$}
\label{app:099}
We redo experiments with $\beta = 0.99$.
Overall, the trend is the same, with the effect of momentum more amplified.
We provide the necessary details and some discussions for each experiment redone.

\subsubsection{Linear Diagonal Networks}
\label{app:beta=0.99-ldn}
Here, we redo the experiments of Section~\ref{sec:ldn} with $\beta = 0.99$, where the results are reported in Figure~\ref{fig:LDN-0.99}.
Note how we expanded the range of $\alpha$'s to see the effect of momentum, which seems to be a bit ``delayed''.
But, at the same time, there is little instability in the trend in that once the curve reaches zero test loss, it stays there; this is in contrast to our $\beta = 0.9$ experiment (Figure~\ref{fig:LDN-main-phb} of Section~\ref{sec:ldn}), where there were some instabilities over $\alpha$'s.

\begin{figure}[!ht]
\centering
\subfigure[PHB Test Loss]{
\includegraphics[width=0.31\textwidth]{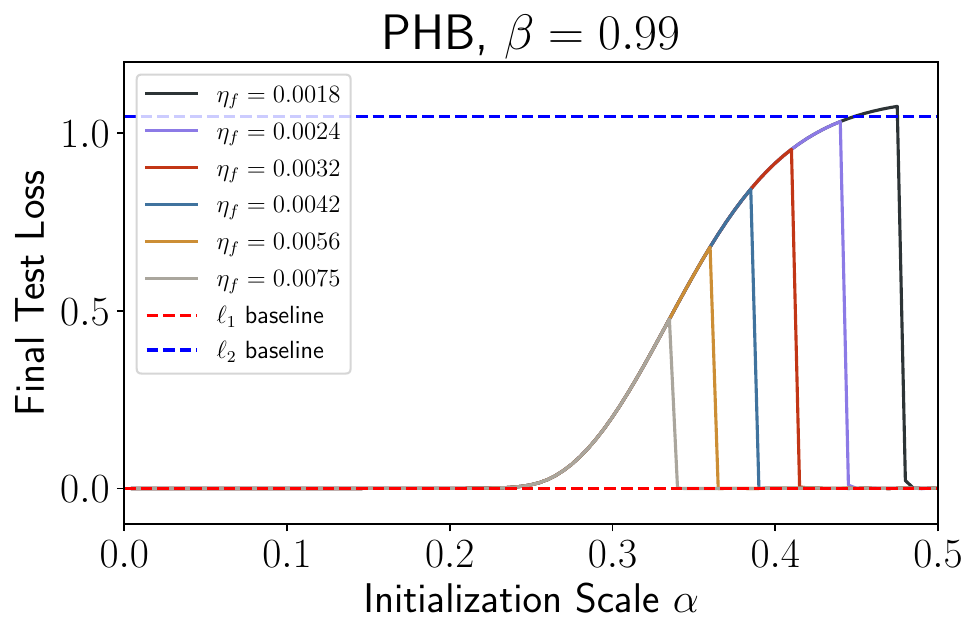}
\label{fig:LDN-testloss-0.99}
}
\subfigure[PHB Sharpness]{
\includegraphics[width=0.31\textwidth]{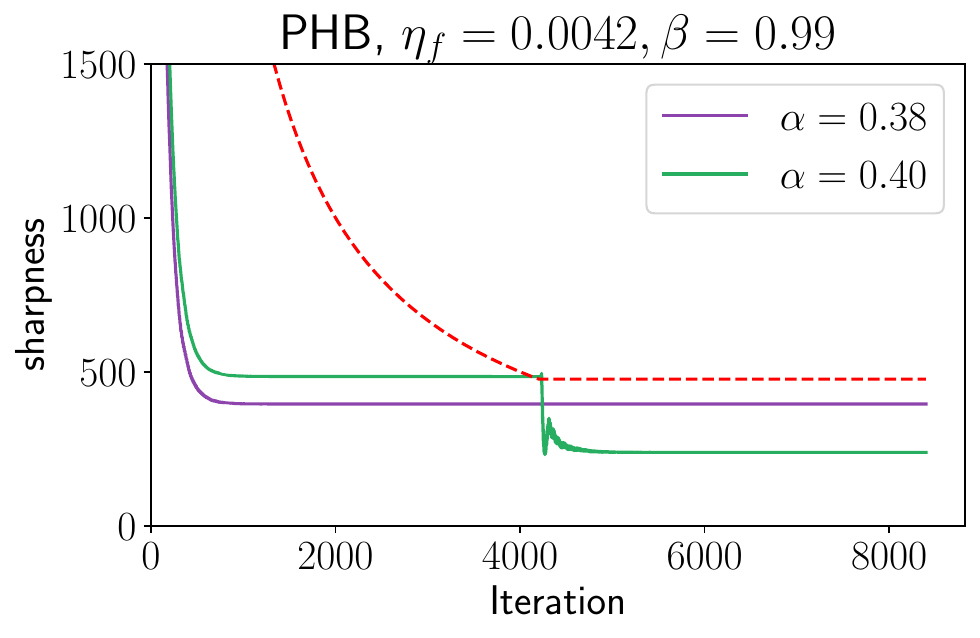}
\label{fig:LDN-sharpness-0.99}
}
\subfigure[PHB Log Train Loss]{
\includegraphics[width=0.31\textwidth]{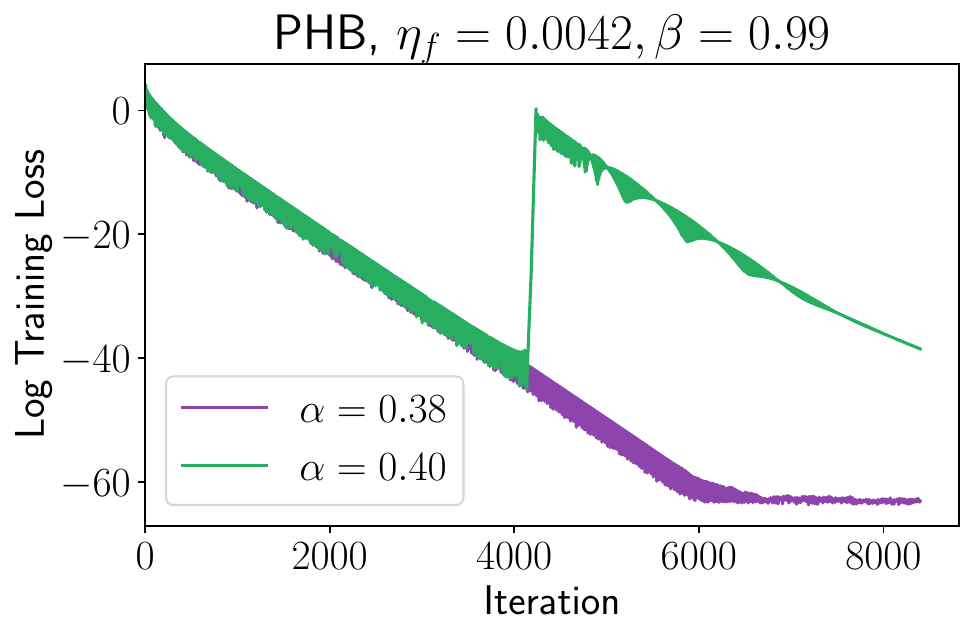}
\label{fig:LDN-trainloss-0.99}
}

\caption{Recall from Section~\ref{sec:ldn} that ``$\ell_1$ baseline'' and ``$\ell_2$ baseline'' in Figure~\ref{fig:LDN-testloss-0.99} stand for the solution with the minimal $\ell_1$ norm and the solution with the minimal $\ell_2$ norm to the regression problem, respectively.}
\label{fig:LDN-0.99}
\end{figure}

% \subsubsection{Toy Example}
% Here, we provide additional results for the toy example with $\beta = 0.99$ for vanilla GD and vanilla PHB.
% The results are shown in Figures~\ref{fig:toy-eta-beta} and \ref{fig:toy-sharpness-beta}.

% \begin{figure}[!h]
% \centering
% \includegraphics[width=0.3\textwidth]{toy-beta=0.99/toy-trajectory-5.pdf}
% \includegraphics[width=0.3\textwidth]{toy-beta=0.99/toy-trajectory-1e1.pdf}
% \includegraphics[width=0.3\textwidth]{toy-beta=0.99/toy-trajectory-1e2.pdf}
% \includegraphics[width=0.3\textwidth]{toy-beta=0.99/toy-trajectory-1e3.pdf}
% \includegraphics[width=0.3\textwidth]{toy-beta=0.99/toy-trajectory-1e5.pdf}
% \includegraphics[width=0.3\textwidth]{toy-beta=0.99/toy-trajectory-1e10.pdf}
% \caption{Trajectories of GD and PHB for the toy example.}
% \label{fig:toy-eta-beta}
% \end{figure}

% \begin{figure}[!h]
% \centering
% \includegraphics[width=0.3\textwidth]{toy-beta=0.99/toy-sharpness-5.pdf}
% \includegraphics[width=0.3\textwidth]{toy-beta=0.99/toy-sharpness-1e1.pdf}
% \includegraphics[width=0.3\textwidth]{toy-beta=0.99/toy-sharpness-1e2.pdf}
% \includegraphics[width=0.3\textwidth]{toy-beta=0.99/toy-sharpness-1e3.pdf}
% \includegraphics[width=0.3\textwidth]{toy-beta=0.99/toy-sharpness-1e5.pdf}
% \includegraphics[width=0.3\textwidth]{toy-beta=0.99/toy-sharpness-1e10.pdf}
% \caption{Sharpness plots of GD and PHB for the toy example.}
% \label{fig:toy-sharpness-beta}
% \end{figure}

\subsubsection{Shallow Nonlinear Neural Networks}
\label{app:beta=0.99-nonlinear}
For nonlinear networks, momentum with $\beta=0.99$ has very unstable training dynamics when trained on a small dataset. Here, we show results for $\beta=0.99$ on larger datasets. For the CIFAR10 experiments, we train on (1) a subset of CIFAR10 with 2 classes and \emph{2000} training images and (2) a larger subset of CIFAR10 with 10 classes and \emph{5000} training images. Results for the 2-class CIFAR10 are shown in Figure~\ref{fig:relu-cifar10-2k-fcn}, and results for the 5k subset of CIFAR10 are shown in Figure~\ref{fig:relu-fcn-beta0.99}. An additional observation is that although PS is exhibited after the large catapults when using $\beta=0.9$, no PS occurs after the large catapults when using $\beta=0.99$ For the synthetic rank-2 dataset, we use a Rank2-400-4000 dataset. Results for the synthetic rank-2 dataset are shown in Figure~\ref{fig:rank2-400-4000}. 

\begin{figure}[!ht]
    \centering
        \includegraphics[width=0.31\textwidth]{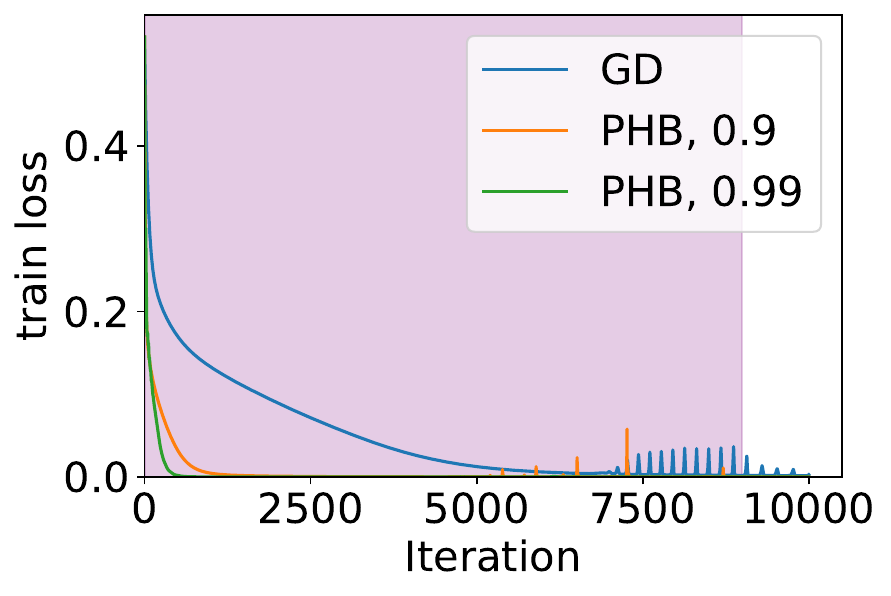}
        \includegraphics[width=0.31\textwidth]{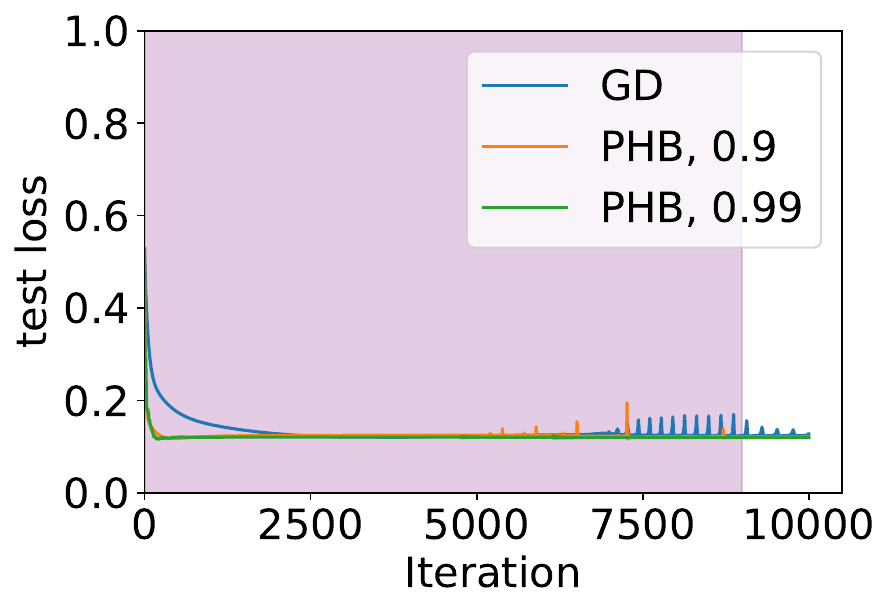}
        \includegraphics[width=0.31\textwidth]{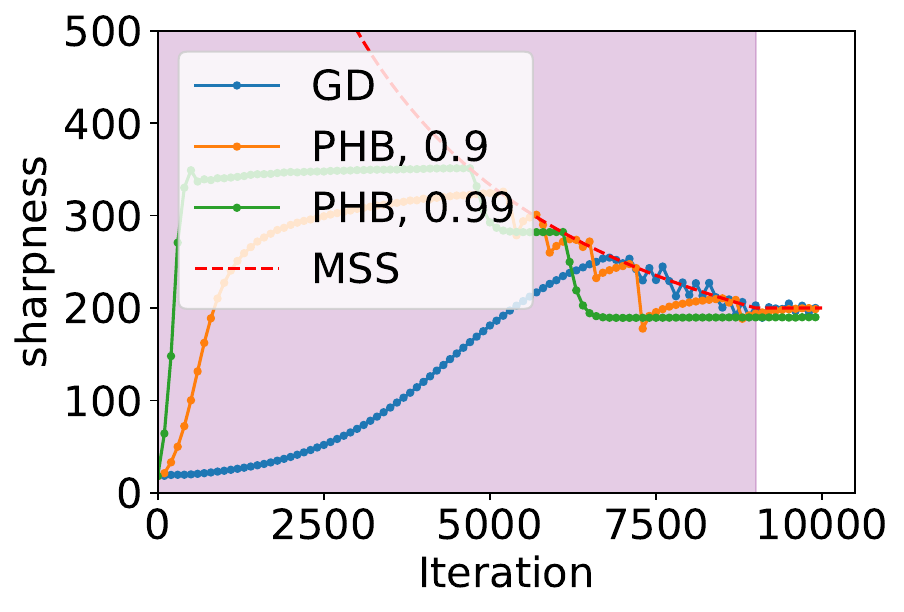}
    \caption{3-layer width-256 FCN trained on a 2k-datapoint subset of CIFAR10 with MSE loss, $\eta_i=0.001$, $\eta_f=0.01$ and 9000 steps of warmup.}
    \label{fig:relu-cifar10-2k-fcn}
\end{figure}

\begin{figure}
    \subfigure[$\eta_i=0.001$, $\eta_f=0.01$]
    {
    \includegraphics[width=0.31\textwidth]{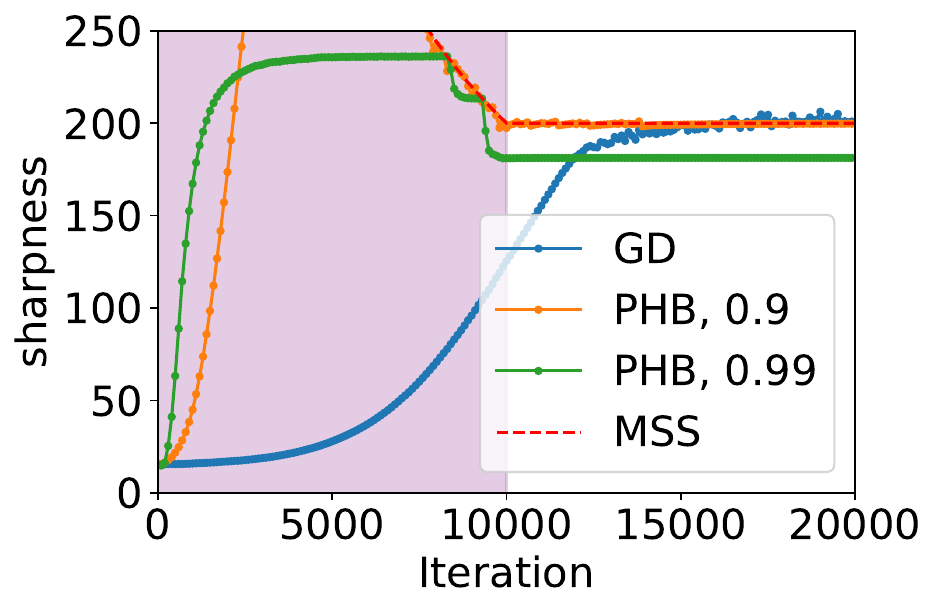}
    }
    \hfill
    \subfigure[$\eta_i=0.001$, $\eta_f=0.03$]
    {
    \includegraphics[width=0.31\textwidth]{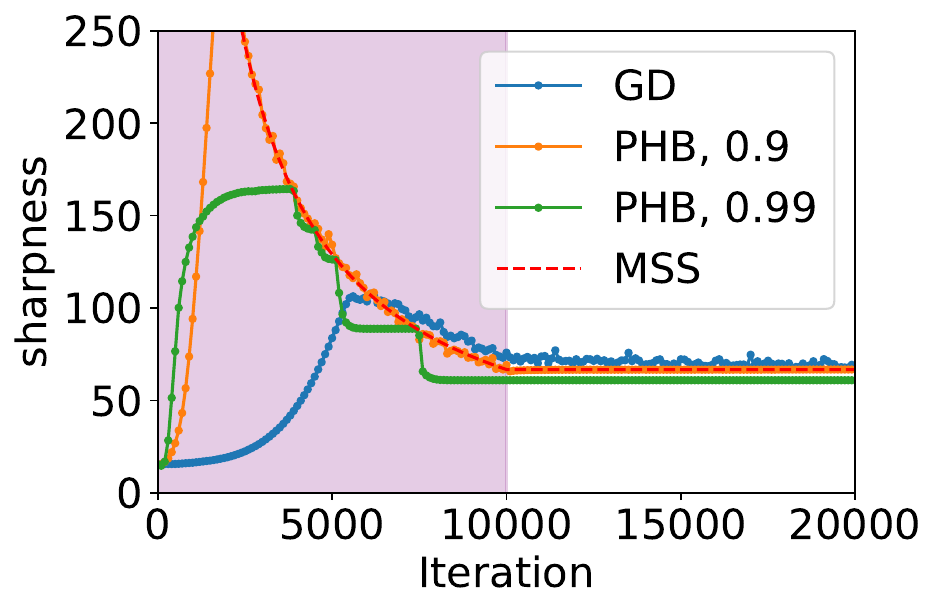}
    }
    \hfill
    \subfigure[$\eta_i=0.001$, $\eta_f=0.05$]
    {
    \includegraphics[width=0.31\textwidth]{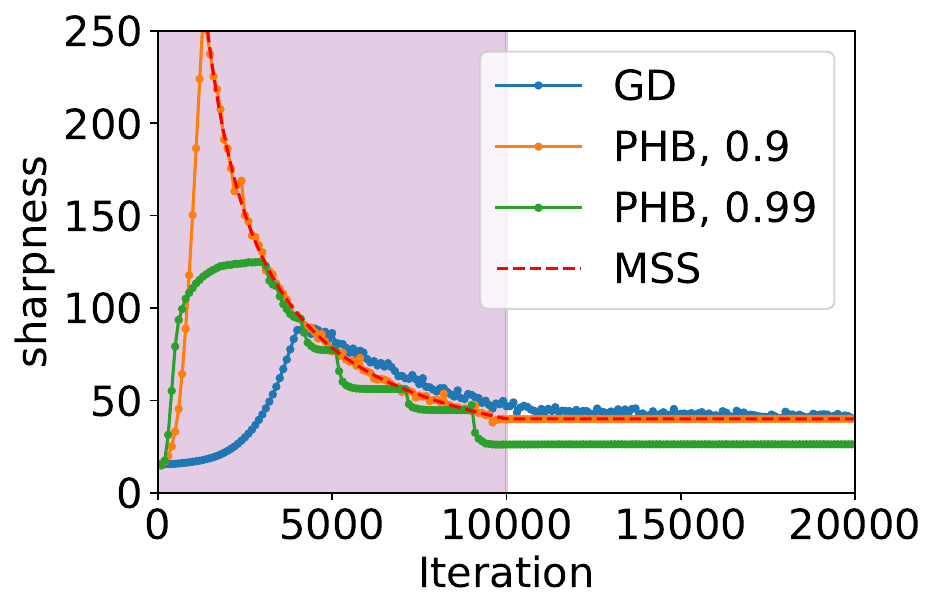}
    }
    \caption{3-layer width-200 FCN trained on a 5k-datapoint subset of CIFAR10 using MSE loss and 10000 steps of warmup}
    \label{fig:relu-fcn-beta0.99}
\end{figure}

\begin{figure}[!h]
    \centering
    \includegraphics[width=0.31\textwidth]{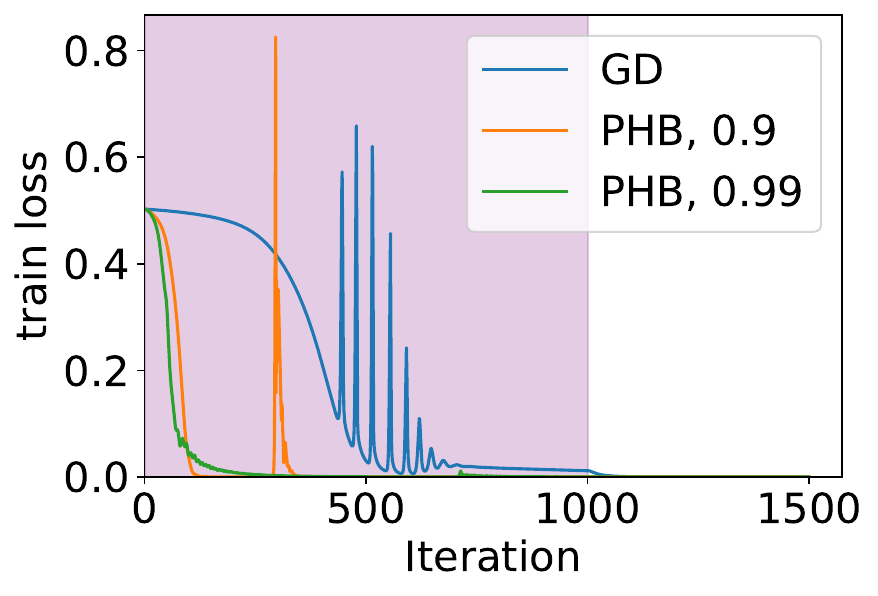}
    \includegraphics[width=0.31\textwidth]{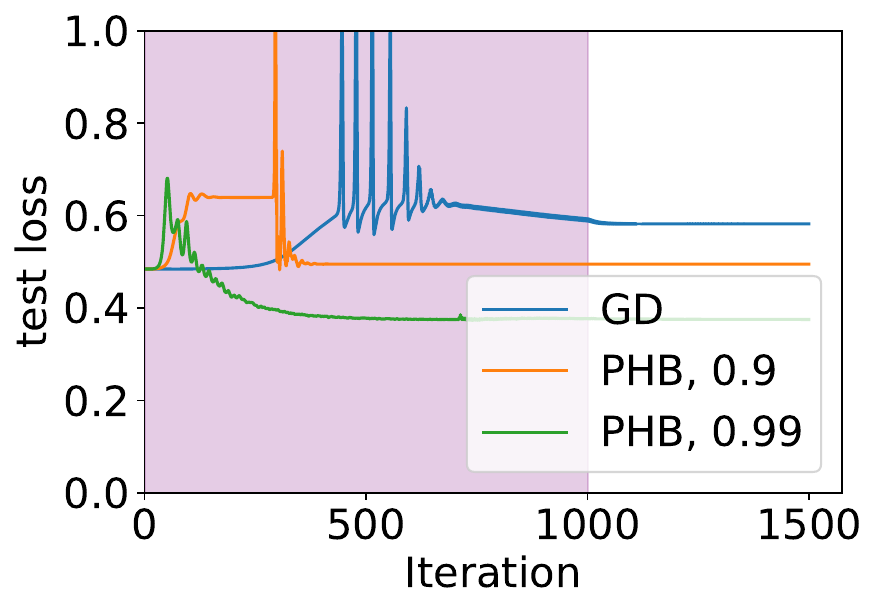}
    \includegraphics[width=0.31\textwidth]{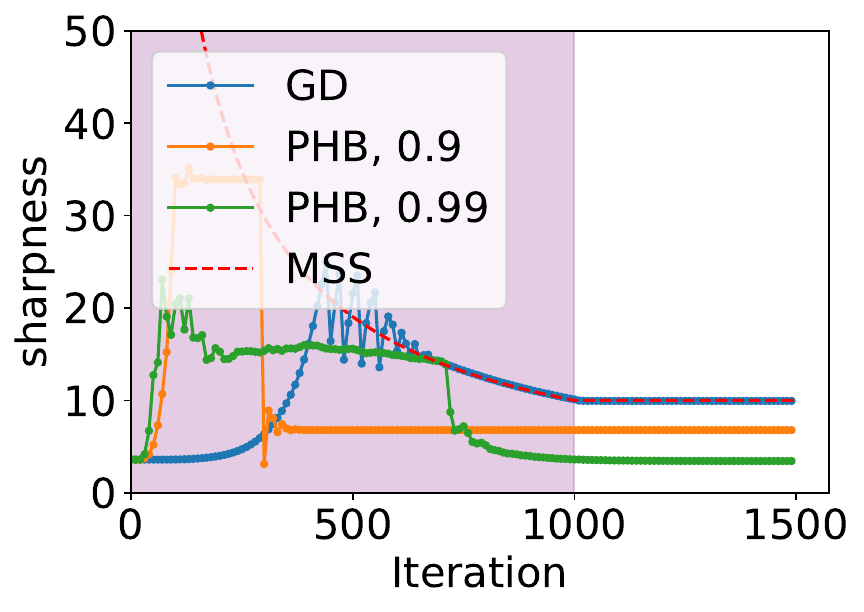}
    \caption{Rank2-400-4000, width=128, MSE loss, $\eta_i=0.01$, $\eta_f=0.2$}
    \label{fig:rank2-400-4000}
\end{figure}

% \newpage
\subsubsection{ResNet20}
The catapults are more pronounced when training ResNet20 using momentum with $\beta=0.99$ as shown in Figure \ref{fig:resnet20-beta0.99}.

\begin{figure}
    \centering
    \subfigure[ResNet20 trained on a 5k-datapoint subset of CIFAR10 with $\eta_i=0.01$, $\eta_f=0.1$, and 5000 steps of warmup]
    {
    \includegraphics[width=0.31\textwidth]{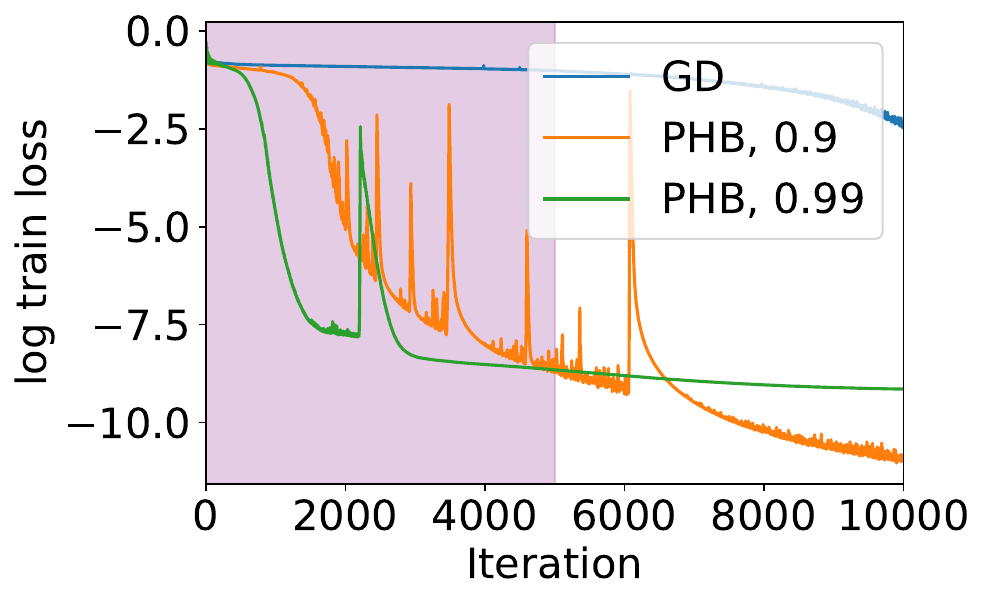}
    \hfill
    \includegraphics[width=0.31\textwidth]{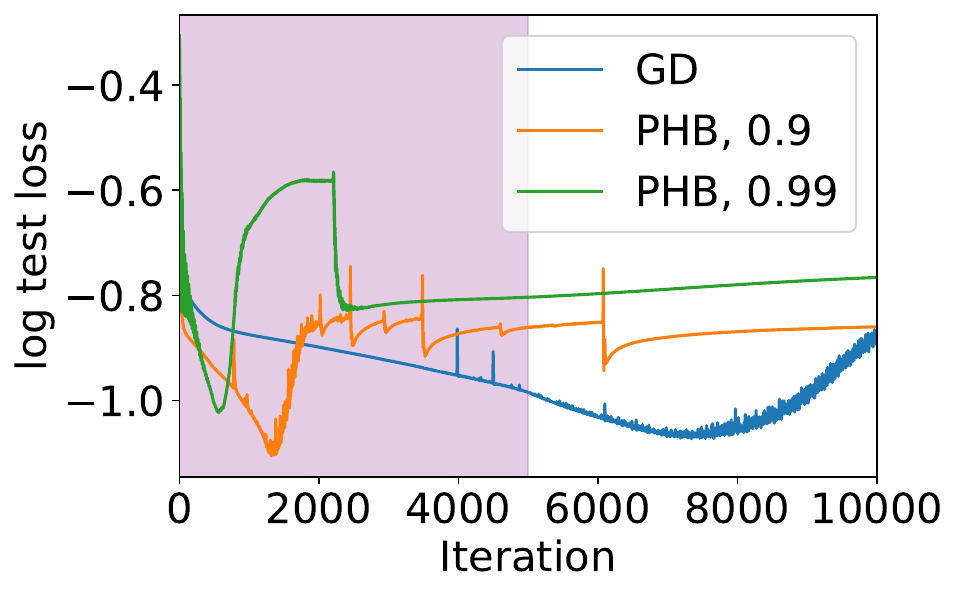}
    \hfill
    \includegraphics[width=0.31\textwidth]{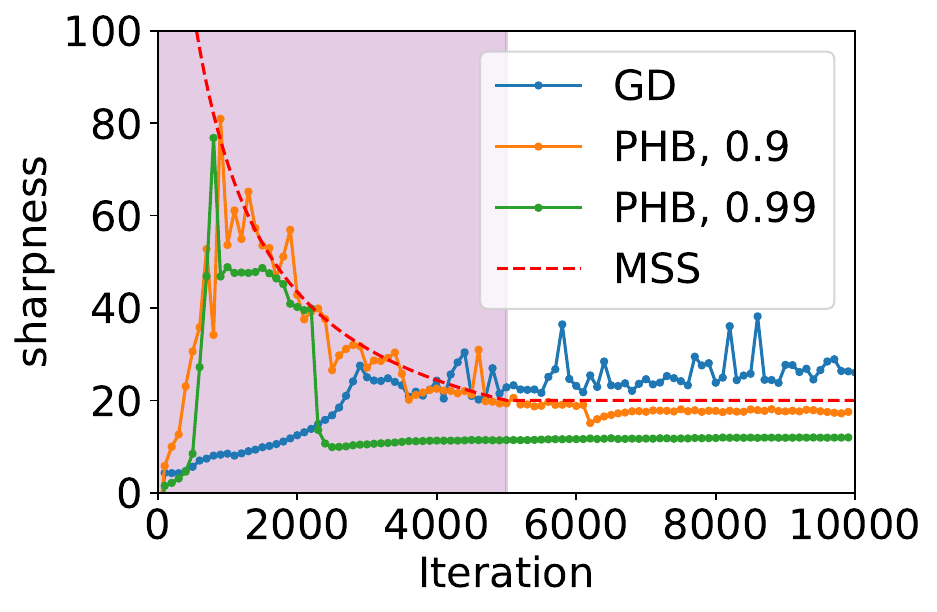}
    }
    \caption{ResNet20 Experiments with $\beta=0.99$}
    \label{fig:resnet20-beta0.99}
\end{figure}

\clearpage
\subsection{Details for Sharpness Displacement Lower bound Experiment for LDN}
\label{sec:LDNLBdetails}

In Section~\ref{sec:ldn-hypothesis}, we discussed extending Theorem~\ref{thm:toy} to more general loss functions, and presented a numerically computed ``lower bound'' on the sharpness decrease for a simple LDN loss $\gL(u,v) = \frac{1}{2}(u^2-v^2-1)^2$. This section provides more details of this process.

For $\gL(u,v) = \frac{1}{2}(u^2-v^2-1)^2$, its gradient and Hessian are given as
\begin{equation}
\label{eq:LDNlossgH}
    \nabla \gL(u,v) = 
    \begin{bmatrix}
    2(u^2-v^2-1)u \\ -2(u^2-v^2-1)v
    \end{bmatrix},\quad
    \nabla^2 \gL(u,v) = 
    \begin{bmatrix}
        6u^2 - 2v^2-2 & -4uv \\
        -4uv & 6v^2 - 2u^2 + 2
    \end{bmatrix}.
\end{equation}
Now consider running gradient flow (GF), whose dynamics is given as
\begin{equation*}
    \begin{bmatrix}
        \dot u(t) \\ \dot v(t) 
    \end{bmatrix}
    =
    - \nabla \gL(u(t),v(t)),
\end{equation*}
starting from $(u(0),v(0))$. From the gradient values, it can be checked that $u(t)v(t)$ stays constant throughout the GF trajectory:
\begin{equation*}
    \frac{d}{dt} (u(t) v(t)) = u(t) \dot v(t) + \dot u(t) v(t) = 0.
\end{equation*}
Hence, the GF trajectory satisfies $u(t) v(t) = u(0) v(0)$ for all $t$, from which we can exactly calculate the solution of GF.

For simplicity, we focus on the case $u(0) > 0$ and $\gL(u(0), v(0)) < 1/2$. In this case, since $\gL(u(t),v(t))$ is always non-increasing along the trajectory, the trajectory has to stay in the region $u(t) > 0$ forever. In such a case, the limit of GF is given by solving
\begin{equation*}
    u(\infty)^2 - \frac{u(0)^2 v(0)^2}{u(\infty)^2} = 1,
\end{equation*}
which amounts to the solution
\begin{equation}
\label{eq:LDNGFsol}
    u(\infty) = \sqrt{\frac{1+\sqrt{1+4u(0)^2v(0)^2}}{2}},~~
    v(\infty) = \frac{u(0) v(0)}{u(\infty)}.
\end{equation}

Now, consider a global minimum $\bm\theta_* = (u_*, v_*)$ of $\gL(u,v)$. The minimum necessarily satisfies $u_*^2 = v_*^2 + 1$. Substituting this to the loss Hessian~\eqref{eq:LDNlossgH} gives
\begin{align*}
    \nabla^2 \gL(u_*,v_*) 
    &= 
    \begin{bmatrix}
        4v_*^2 + 4 & -4v_* \sqrt{v_*^2+1} \\
        -4v_* \sqrt{v_*^2+1} & 4v_*^2
    \end{bmatrix}\\
    &=
    (8v_*^2 + 4)
    \begin{bmatrix}
        \frac{\sqrt{v_*^2+1}}{\sqrt{2v_*^2+1}}\\
        -\frac{v}{\sqrt{2v_*^2+1}}
    \end{bmatrix}
    \begin{bmatrix}
        \frac{\sqrt{v_*^2+1}}{\sqrt{2v_*^2+1}} &
        -\frac{v}{\sqrt{2v_*^2+1}}
    \end{bmatrix}
\end{align*}
from which we can see that
\begin{equation*}
S(\bm\theta_*) = 8v_*^2 + 4,\quad
\vw_{\max} (\bm\theta_*) = \begin{bmatrix}
        \frac{\sqrt{v_*^2+1}}{\sqrt{2v_*^2+1}}\\
        -\frac{v}{\sqrt{2v_*^2+1}}
    \end{bmatrix},
\end{equation*}
which are the key quantities need for the calculation of $C_u$ and $C_v$~\eqref{eq:CuCv2}.

Based on this background, we draw Figure~\ref{fig:ldn-momentum-on-off} using the following procedure. We start GD/PHB from $(u_0,v_0)$, whose GF solution $(u_0^*,v_0^*)$ has sharpness $\frac{2+\epsilon}{\eta}$. Specifically, we choose $\eta = 0.01$ and $\epsilon = 0.004$, and initialize at $(u_0,v_0) \approx (5.060, 4.950)$. Every time we update the iterates to get $\bm\theta_t = (u_t,v_t)$, we calculate the corresponding GF solution $\bm\theta_t^* = (u_t^*,v_t^*)$ using \eqref{eq:LDNGFsol}. From there, we calculate the sharpness, and see if $S(\bm\theta_t^*) < \frac{2-\epsilon}{\eta}$; let $\tau_u$ be the first time step $t$ that $S(\bm\theta_t^*) < \frac{2-\epsilon}{\eta}$ happens; from this, we can calculate $C_u$ and $C_v$ as defined in \eqref{eq:CuCv2} until convergence. In our experiments, we observed that $S(\bm\theta_t^*)$ was monotonically non-increasing, so there was no subtlety involved in calculating $C_u$ and $C_v$. The monotone decrease of sharpness of GF solution is consistent with \citet{kreisler2023gradient}.

\end{document}